%% file: DTW_Mean.tex
\documentclass[DIV=calc,paper=a4,fontsize=10pt]{scrartcl}

\include{config}
\usepackage{framed}

\title{Nonsmooth Analysis and Subgradient Methods\\ for Averaging in Dynamic Time Warping Spaces} 

\author{David Schultz and Brijnesh Jain \\ Technische Universit\"at Berlin, Germany} 

\date{} 

\begin{document}

\maketitle 

\paragraph*{Abstract.}
Time series averaging in dynamic time warping (DTW) spaces has been successfully applied to improve pattern recognition systems. This article proposes and analyzes subgradient methods for the problem of finding a sample mean in DTW spaces. The class of subgradient methods generalizes existing sample mean algorithms such as DTW Barycenter Averaging (DBA). We show that DBA is a majorize-minimize algorithm that converges to necessary conditions of optimality after finitely many iterations.  Empirical results show that for increasing sample sizes the proposed stochastic subgradient (SSG) algorithm is more stable and finds better solutions in shorter time than the DBA algorithm on average. Therefore, SSG is useful in online settings and for non-small sample sizes. The theoretical and empirical results open new paths for devising sample mean algorithms: nonsmooth optimization methods and modified variants of pairwise averaging methods.

\clearpage

\begin{framed}
\begin{small}
\tableofcontents
\end{small}
\end{framed}

\clearpage

\section{Introduction}
Inspired by the sample mean in Euclidean spaces, the goal of time series averaging is to construct a time series that is located in the center of a given sample of time series. One common technique to average time series is based on first aligning the time series and then synthesizing the aligned time series to an average. In doing so, the alignment of time series is typically performed with respect to the dynamic time warping (DTW) distance. Several variations of this approach have been successfully applied to improve nearest neighbor classifiers and to formulate centroid-based clustering algorithms in DTW spaces \cite{Abdulla2003,Hautamaki2008,Oates1999,Petitjean2016,Rabiner1979,Soheily-Khah2016}. 

Among the different variations, one research direction poses time series averaging as an optimization problem \cite{Hautamaki2008,Petitjean2011,Soheily-Khah2016}: Suppose that $\S{X} = \args{x^{(1)}, \dots, x^{(N)}}$ is a sample of $N$ time series $x^{(i)}$. Then a (sample) mean in DTW spaces is any time series that minimizes the Fr\'echet function \cite{Frechet1948}
\[
F(x) = \frac{1}{N}\sum_{k=1}^N \dtw^2\!\args{x, x^{(k)}},
\]
where $\dtw$ is the DTW distance. A polynomial-time algorithm for finding a global minimum of the non-differentiable, non-convex Fr\'echet function is unknown. Therefore, an ongoing research problem is to devise improved approximation algorithms that efficiently find acceptable sub-optimal solutions. 
Currently, the most popular method for minimizing the Fr\'echet function is the DTW Barycenter Averaging (DBA) algorithm proposed by Petitjean et al.~\cite{Petitjean2011}. Empirical results have shown that the DBA algorithm outperforms competing mean algorithms \cite{Petitjean2011,Soheily-Khah2015}. 

This article proposes and investigates subgradient methods for minimizing the Fr\'echet function. The main contributions are as follows:
\begin{enumerate}[(i)]
\itemsep0em
\item A stochastic subgradient mean algorithm that outperforms DBA for non-small sample sizes.
\item Necessary and sufficient conditions of optimality.
\item Finite convergence of  DBA  to solutions satisfying the necessary conditions of optimality.
\end{enumerate} 
Subgradient methods are nonsmooth optimization \cite{Bagirov2014} techniques that operate very similar to gradient descent methods, but replace the gradient with a subgradient. The concept of subgradient serves to generalize gradients under mild conditions that hold for the Fr\'echet function.
We propose two subgradient methods, the majorize-minimize mean (MM) algorithm and the stochastic subgradient mean (SSG) algorithm. We show that the MM algorithm is equivalent to the DBA algorithm. 
Formulating DBA as a nonsmooth optimization method following a majorize-minimize \cite{Hunter2004} principle considerably simplifies its analysis and comparison to SSG.

Both algorithms DBA and SSG are iterative methods that repeatedly update a candidate solution. The main difference between both algorithms is that DBA is a batch and SSG a stochastic optimization method: While the DBA algorithm computes an exact subgradient on the basis of all sample time series, the SSG algorithm estimates a subgradient of the Fr\'echet function on the basis of a single randomly picked sample time series. Consequently, every time the algorithms have processed the entire sample, the SSG algorithm has performed $N$ updates, whereas the DBA algorithm has performed a single update. 
The different update rules of both subgradient methods have the following implications:
\begin{enumerate}
\item Theoretical implication: 
The SSG algorithm is not a descent method. During optimization, the value of the Fr\'echet function can increase. In contrast, Petitjean et al.~\cite{Petitjean2016} proved that DBA is a descent method. Moreover, we show that DBA converges to solutions satisfying necessary conditions of optimality after a finite number of updates. Necessary conditions of optimality characterize the form of local minimizers of the Fr\'echet function. 
If a solution satisfies the sufficient conditions, we can conclude local minimality.

\item Practical implication: Empirical results suggest that the SSG algorithm is more stable and finds better solutions in substantially less time than the DBA algorithm provided that the sample size $N$ is sufficiently large, that is $N>50$ as a rule of thumb. 
\end{enumerate}
In summary, the DBA algorithm has stronger theoretical properties than the SSG algorithm, whereas SSG exhibits superior performance than DBA for non-small sample sizes.

\medskip

The rest of the article is structured as follows: Section \ref{sec:background_and_related_work} provides background material and discusses related work. In Section \ref{sec:properties}, we study analytical properties of the minimization problem. Section \ref{sec:algorithms} proposes subgradient methods for minimizing the Fr\'echet function. In Section \ref{sec:experiments}, we present and discuss empirical results. Finally, Section \ref{sec:conclusion} concludes. Proofs are delegated to the appendix.

\section{Background and Related Work}\label{sec:background_and_related_work}

This section introduces the DTW distance, the Fr\'echet function, and finally reclassifies existing mean algorithms.

\subsection{The Dynamic Time Warping Distance}

\begin{figure}
\centering
\includegraphics[width=0.98\textwidth]{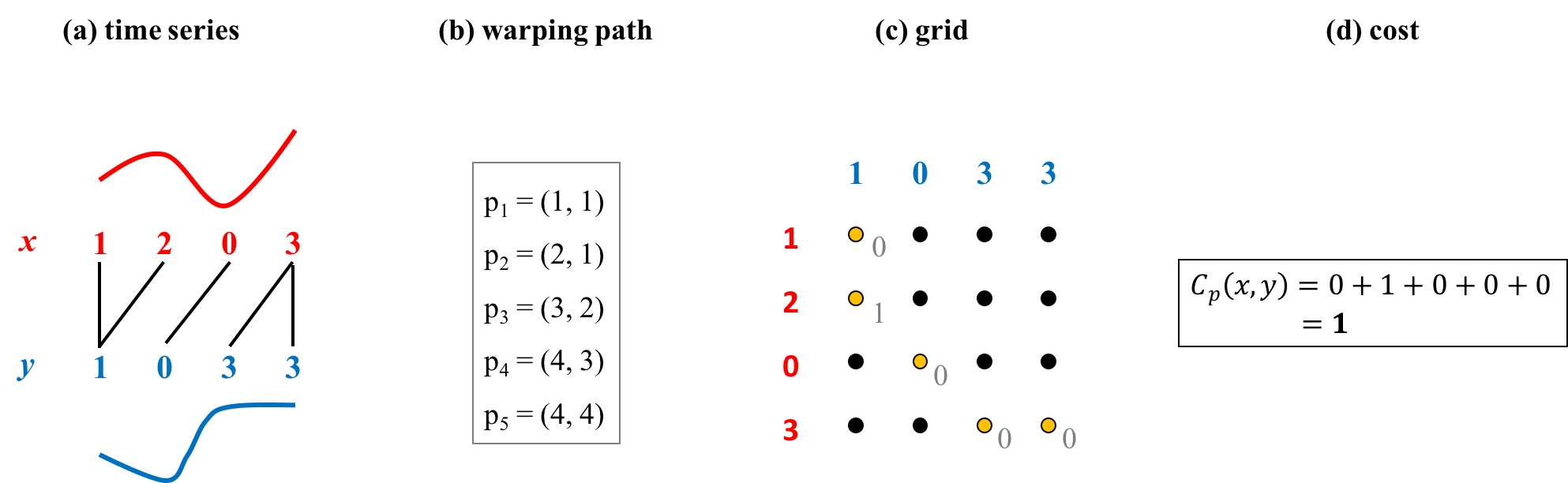}
\caption{\textbf{(a)} Two time series $x$ and $y$ of length $n = 4$. The red and blue numbers are the elements of the respective time series. \textbf{(b)} Warping path $p = (p_1, \dots, p_5)$ of length $L = 5$. The points $p_l = (i_l, j_l)$ of warping path $p$ align elements $x_{i_l}$ of $x$ to elements $y_{j_l} $ of $y$ as illustrated in (a) by black lines. \textbf{(c)} The $4 \times 4$ grid showing how warping path $p$ moves from the upper left to the lower right corner as indicated by the orange balls. The numbers attached to the orange balls are the squared-error costs of the corresponding aligned elements. \textbf{(d)} The cost $C_p(x, y)$ of aligning $x$ and $y$ along warping path $p$.}
\label{fig:wpath}
\end{figure}

We begin with introducing the DTW distance and refer to Figure \ref{fig:wpath} for illustrations of the concepts. 

\medskip

A \emph{time series} $x$ of \emph{length} $m$ is an ordered sequence $x = (x_1, \dots, x_m)$ consisting of \emph{elements} $x_i \in \R^d$ for every \emph{time point} $i \in [m]$, where we denote $[m] = \{1, \dots, m \}$. A time series is said to be \emph{univariate} if $d = 1$ and \emph{multivariate} if $d > 1$. For a fixed $d \in \N$, we denote by $\S{T}_*$ the set of all time series of finite length and by $\S{T}_m$ the set of all time series of length $m \in \N$. The DTW distance is a distance function on $\S{T}_*$ based on the notion of warping path.
\begin{definition}
Let $m, n \in \N$. A \emph{warping path} of order $m \times n$ is a sequence $p = (p_1 , \dots, p_L)$ of $L$ points $p_l = (i_l,j_l) \in [m] \times [n]$ such that
\begin{enumerate}
\item $p_1 = (1,1)$ and $p_L = (m,n)$ \hfill\emph{(\emph{boundary conditions})}
\item $p_{l+1} - p_{l} \in \cbrace{(1,0), (0,1), (1,1)}$ for all $l \in [L-1]$ \hfill\emph{(\emph{step condition})} 
\end{enumerate}
\end{definition}
The set of all warping paths of order $m \times n$ is denoted by $\S{P}_{m,n}$. A warping path of order $m \times n$ can be thought of as a path in a $[m] \times [n]$ grid, where rows are ordered top-down and columns are ordered left-right. The boundary condition demands that the path starts at the upper left corner and ends in the lower right corner of the grid. The step condition demands that a transition from point to the next point moves a unit in exactly one of the following directions: down, diagonal, and right. 

A warping path $p = (p_1, \dots, p_L)\in \S{P}_{m,n}$ defines an alignment (or warping) between time series $x = (x_1, \dots, x_m)$ and $y = (y_1, \dots, y_n)$. Every point $p_l = (i_l,j_l)$ of warping path $p$ aligns element $x_{i_l}$ to element $y_{j_l}$. The \emph{cost} of aligning time series $x$ and $y$ along warping path $p$ is defined by
\begin{equation*}
C_p(x,y) = \sum_{l=1}^L \normS{x_{i_l} - y_{j_l}}{^2},
\end{equation*}
where the Euclidean norm $\norm{\cdot}$ on $\R^d$ is used as \emph{local cost function}. Then the DTW distance between two time series minimizes the cost of aligning both time series over all possible warping paths. 

\begin{definition}
Let $x$ and $y$ be two time series of length $m$ and $n$, respectively. The \emph{DTW distance} between $x$ and $y$ is defined by
\begin{equation*}
\dtw(x,y) = \min \cbrace{\sqrt{C_p(x,y)} \,:\, p \in \S{P}_{m,n}}.
\end{equation*}
An \emph{optimal warping path} is any warping path $p \in \S{P}_{m,n}$ satisfying $\dtw(x, y) = \sqrt{C_p(x,y)}$.
\end{definition}
A \emph{DTW space} is any subset $\S{T} \subseteq \S{T}_*$ endowed with the DTW distance. In particular, $\S{T}_*$ and $\S{T}_m$ are DTW spaces. Even if the underlying local cost function is a metric, the induced DTW distance is generally only a pseudo-semi-metric satisfying 
\begin{enumerate}
\item $\dtw(x, y) \geq 0$
\item $\dtw(x, x) = 0$
\end{enumerate}
for all $x, y \in \S{T}_*$. Computing the DTW distance and deriving an optimal warping path is usually solved by applying techniques from dynamic programming \cite{Sakoe1978}.

\subsection{Fr\'echet Functions}

Throughout this contribution, we consider a restricted form of the Fr\'echet function. 

\begin{definition}
Let $n \in \N$. Suppose that $\S{X} = \args{x^{(1)}, \dots, x^{(N)}}$ is a sample of $N$ time series $x^{(k)} \in \S{T}_*$. Then the function 
\begin{align*}
F: \S{T}_n \rightarrow \R, \quad x \mapsto \frac{1}{N}\sum_{i = 1}^N \dtw^2\args{x,x^{(k)}}
\end{align*}
is the \emph{Fr\'echet function} of sample $\S{X}$. The value $F(x)$ is the \emph{Fr\'echet variation} of $\S{X}$ at $x \in \S{T}_n$.
\end{definition}

The domain $\S{T}_n$ of the Fr\'echet function is restricted to the subset of all time series of fixed length 
$n$. In contrast, the sample time series $x^{(1)}, \dots, x^{(N)}$ can have arbitrary and different lengths.
\begin{definition}
The \emph{sample mean set} of $\S{X}$ is the set 
\[
\S{F} = \cbrace{z \in \S{T}_n \,:\, F(z) \leq F(x) \text{ for all } x \in \S{T}_n}.
\]
Each element of $\S{F}$ is called \emph{(sample) mean} of $\S{X}$. 
\end{definition}
A sample mean is a time series that minimizes the Fr\'echet function $F$. A sample mean of $\S{X}$ exists \cite{Jain2016b}, but is not unique, in general. We refer to the problem of minimizing the Fr\'echet function as the sample mean problem.

\subsection{Related Work}\label{sec:related-work}

The majority of time series averaging methods reported in the literature can be classified into two categories: (1) asymmetric-batch methods, and (2) symmetric-incremental methods. The asymmetric-symmetric dimension determines the type of average and the batch-incremental dimension determines the strategy with which a sample of time series is synthesized to one of both types of averages. We neither found work on symmetric-batch nor on asymmetric-incremental methods. There is a well-founded explanation for the absence of symmetric-batch algorithms, whereas asymmetric-incremental algorithms apparently have not been considered as a possible alternative to existing methods. A third category that has scarcely been applied to time series averaging are meta-heuristics. An exception are genetic algorithms proposed by Petitjean et al.~\cite{Petitjean2012}. In this section, we discuss existing methods within the symmetric-asymmetric and batch-incremental dimensions.

\subsubsection{Asymmetric vs.~Symmetric Averages}
The asymmetric-symmetric dimension defines the form of an average. To describe the different forms of an average, we assume that $p$ is an optimal warping path between time series $x$ and $y$. 

\paragraph*{Asymmetric Averages.}
An asymmetric average of $x$ and $y$ is computed as follows: (i) select a reference time series, say $x$; (ii) compute an optimal warping path $p$ between $x$ and $y$; and (iii) average $x$ and $y$ along path $p$ with respect to the time axis of reference $x$.

The different variants of the third step \cite{Abdulla2003,Petitjean2011,Rabiner1979} have the following properties in common \cite{Kruskal1983}: (i) they admit simple extension to averaging several time series; and (ii) the length of the resulting average coincides with the length of the reference $x$. Therefore, the form of an asymmetric average depends on the choice of the reference time series. Hence, an asymmetric average for a given optimal warping path between two time series is not well-defined, because there is no natural criterion for choosing a reference. 

\paragraph*{Symmetric Averages.}
According to Kruskal and Liberman \cite{Kruskal1983}, the symmetric average of time series $x$ and $y$ along an optimal warping path $p$ is a time series $z$ of the same length as warping path $p$ and consists of elements $z_l = \args{x_{i_l} + y_{j_l}}/2$ for all points $p_l = (i_l, j_l)$ of $p$. Not only the elements $x_i$ and $y_j$ at the warped time points $i$ and $j$ but also the warped time points $i$ and $j$ themselves can be averaged. 

Symmetric averages $z$ of time series $x$ and $y$ are well-defined for a given optimal warping path $p$, but generally have more time points (a finer sampling rate) than $x$ and $y$. Kruskal and Liberman \cite{Kruskal1983} pointed to different methods for adjusting the sampling rate of $z$ to the sampling rates of $x$ and $y$, which results in averages whose length is adapted to the length of $x$ and $y$.

\subsubsection{Batch vs.~Incremental Averaging}

The batch-incremental dimension describes the strategy for combining more than two time series to an average.

\paragraph*{Batch Averaging.}

Batch averaging first warps all sample time series into an appropriate form and then averages the warped time series. We distinguish between asymmetric-batch and symmetric-batch methods. 

The \emph{asymmetric-batch} method presented by Lummis \cite{Lummis1973} and Rabiner \& Wilpon \cite{Rabiner1979} first chooses an initial time series $z$ as reference. Then the following steps are repeated until termination: (i) warp all time series onto the time axis of the reference $z$; and (ii) assign the average of the warped time series as new reference $z$. By construction, the length of the reference $z$ is identical in every iteration. Consequently, the final solution (reference) of the asymmetric-batch method depends on the choice of the initial solution. 

Since the early work in the 1970ies, different variants of the asymmetric-batch method have been proposed and applied. Oates et al.~\cite{Oates1999} and Abdulla et al.~\cite{Abdulla2003} applied an asymmetric-batch method confined to a single iteration. Hautamaki et al.~\cite{Hautamaki2008} completed the approach by \cite{Abdulla2003} by iterating the update step several times. With the DBA algorithm, Petitjean et al.~\cite{Petitjean2011} presented a sound solution in full detail. In the same spirit, Soheily-Khah et al.~\cite{Soheily-Khah2016} generalized the asymmetric-batch method to weighted and kernelized versions of the DTW distance.  

The \emph{symmetric-batch} method suggested by Kruskal and Liberman \cite{Kruskal1983} requires to find an optimal warping in an $N$-dimensional hypercube, where $N$ is the sample size. Since finding an optimal common warping path for $N$ sample time series is computationally intractable, symmetric-batch methods have not been further explored. 

\paragraph*{Incremental Averaging.}
Incremental methods synthesize several sample time series to an average time series by pairwise averaging. The general procedure repeatedly applies the following steps: (i) select a pair of time series $x$ and $y$; (ii) compute average $z$ of time series $x$ and $y$;  (iii) include $z$ into the sample; (iv) optionally remove $x$ and/or $y$ from the sample.  Incremental methods differ in the strategy of selecting the next pair of time series in step (i) and removing time series from the sample in step (iv). Pairwise averaging can take either of both forms, asymmetric and symmetric. 

Kruskal and Liberman \cite{Kruskal1983} presented a general description of symmetric-incremental methods in 1983 that includes progressive approaches as applied in multiple sequence alignment in bioinformatics \cite{Gusfield1997}. The most cited concrete realizations of the Kruskal-Liberman method are weighted averages for self-organizing maps \cite{Somervuo1999}, two \emph{nonlinear alignment and averaging filters} (NLAAF) \cite{Gupta1996}, and the \emph{prioritized shape averaging} (PSA) algorithm \cite{Niennattrakul2009}. For further variants of incremental methods we refer to \cite{Niennattrakul2012,Ongwattanakul2009,Srisai2009}.

Finally, we note that there is apparently no work on asymmetric-incremental methods. The proposed SSG algorithm fills this gap. As we will see later, empirical results suggest to transform existing symmetric-incremental to asymmetric-incremental methods.

\subsubsection{Empirical Comparisons}

In experiments, the asymmetric-batch method DBA outperformed the symmetric-incremental methods NLAAF and PSA \cite{Petitjean2011,Soheily-Khah2015}. This result shifted the focus from symmetric-incremental to asymmetric-batch methods. Since symmetric-batch methods are considered as computationally intractable, the open question is how asymmetric-incremental methods will behave and perform. This question is partly answered in this article.

\section{Analytical Properties of the Sample Mean Problem}\label{sec:properties}

This section first analyzes the local properties of Fr\'echet functions. Then we present necessary and sufficient conditions of optimality. For the sake of clarity, the main text derives all results for the special case of univariate time series of fixed length $n$. The general case of multivariate sample time series of variable length is considered in Section \ref{sec:generalizations}. Therefore, we use the following uncluttered notation in this and the following sections:

\begin{framed}
\noindent
\begin{tabular}{l@{\;$:$\quad}l}
$\S{T}$ & \emph{set of univariate time series $x=(x_1,\dots,x_n)$ of fixed length $n \in \N$, where $x_i \in \R$} \\
$\S{T}^N$ & \emph{$N$-fold Cartesian product, where $N \in \N$}\\
$\S{P}$ & \emph{set $\S{P}_{n,n}$ of all warping paths of order $n \times n$}\\
$\S{P}_*(x, y)$ & \emph{set of optimal warping paths between time series $x$ and $y$}
\end{tabular}
\end{framed}

\subsection{Decomposition of the Fr\'echet Function}\label{subsec:decomposition}

A discussion of analytical properties of the Fr\'echet function is difficult since standard analytical concepts such as locality, continuity, and differentiability are unknown in DTW spaces. 
To approach the sample mean problem analytically, we change the domain of the Fr\'echet function $F$ from the DTW space $\T$ to the Euclidean space $\R^n$. This modification does not change the sample mean set, but local properties of $F$ on different domains may differ. By saying the Fr\'echet function $F:\S{T} \rightarrow \R$ has some analytical property, we tacitly assume its Euclidean characterization.

\medskip

We decompose the Fr\'echet function into a mathematically more convenient form. Suppose that $\S{X} = \args{x^{(1)},\dots,x^{(N)}} \in \S{T}^N$ is a sample of time series. Substituting the definition of the DTW distance into the Fr\'echet function of $\S{X}$ gives
\begin{align}\label{eq:frechet:expanded}
F(x) = \frac{1}{N}\sum_{k=1}^N \dtw^2\!\args{x, x^{(k)}} = \frac{1}{N}\sum_{k=1}^N \;\min_{p^{(k)} \in \S{P}} \;C_{p^{(k)}}\args{x, x^{(k)}},
\end{align}
where $C_p(x, y)$ is the cost of aligning time series $x$ and $y$ along warping path $p$. Interchanging summation and the $\min$-operator in Eq.~\eqref{eq:frechet:expanded} yields
\begin{align}\label{eq:alg:frechet}
F(x) = \min_{p^{(1)} \in \S{P}} \cdots \min_{p^{(N)} \in \S{P}} \frac{1}{N}\sum_{k=1}^N C_{p^{(k)}}\args{x, x^{(k)}}.
\end{align}
To simplify the notation in Eq.~\eqref{eq:alg:frechet}, we introduce configurations and component functions. A \emph{configuration} of warping paths is an ordered list $\S{C} = \args{p^{(1)}, \dots, p^{(N)}} \in \S{P}^N$, where warping path $p^{(k)}$ is associated with time series $x^{(k)}$, $k \in [N]$. A \emph{component function} of $F(x)$ is a function of the form 
\[
F_{\S{C}}: \R^n \rightarrow \R, \quad x \mapsto \frac{1}{N}\sum_{k=1}^N C_{p^{(k)}}\!\args{x, x^{(k)}},
\]
where $\S{C} = \args{p^{(1)}, \dots, p^{(N)}}$ is a configuration. Using the notions of configuration and component function, we can equivalently rewrite the Fr\'echet function as
\begin{align}\label{eq:F_as_min}
F(x) = \min_{\S{C} \in \S{P}^N} F_{\S{C}}(x).
\end{align}
We say, $F_{\S{C}}$ is \emph{active} at $x$ if $F_{\S{C}}(x) = F(x)$. By $\S{A}_F(x)$ we denote the set of active component functions of $F$ at time series $x$. A configuration ${\S{C}}$ is \emph{optimal} at $x$ if $F_{\S{C}}$ is an active component at $x$. In this case, every $p^{(k)} \in \S{C}$ is an optimal warping path between $x$ and $x^{(k)}$.

\subsection{Local Lipschitz Continuity of the Fr\'echet Function}\label{ss:localLipschitz}

By definition of the cost functions $C_p$ we find that every component function $F_{\S{C}}$ is convex and differentiable. Thus, Eq.~\eqref{eq:F_as_min} implies that the Fr\'echet function $F$ is the pointwise minimum of finitely many convex differentiable functions. Since convexity is not closed under $\min$-operations, the Fr\'echet function is non-convex. Similarly, the Fr\'echet function $F$ is not differentiable, because differentiability is also not closed under $\min$-operations. 

We show that the Fr\'echet function $F$ is locally Lipschitz continuous.\footnote{A function $f:\R^n\rightarrow \R$ is locally Lipschitz at point $x \in \R^n$ if there are scalars $L > 0$ and $\varepsilon > 0$ such that $\abs{f(y)-f(z)} \leq L\norm{y-z}$ for all $y, z \in \S{B}(x, \varepsilon) = \cbrace{u \in \R^n \,:\, \norm{u-x} \leq \varepsilon}$.} Local Lipschitz continuity of $F$ follows from two properties: (i) continuously differentiable functions are locally Lipschitz; and (ii) the local Lipschitz property is closed under the $\min$-operation.
 
Any locally Lipschitz function is differentiable almost everywhere by Rademacher's Theorem \cite{Evans1992}. In addition, locally Lipschitz functions admit a concept of generalized gradient at non-differentiable points, called subdifferential henceforth. The subdifferential $\partial F(x)$ of the Fr\'echet function $F$ at point $x$ is a non-empty, convex, and compact set. At differentiable points $x$, the subdifferential $\partial F(x)$ coincides with the gradient $\nabla F(x)$, that is $\partial F(x) = \cbrace{\nabla F(x)}$. At non-differentiable points $x$, we have
\[
\nabla F_{\S{C}}(x) \in \partial F(x) 
\]
for all active component functions $F_{\S{C}} \in \S{A}_F(x)$. The elements $g \in \partial F(x)$ are called the subgradients of $F$ at $x$. We refer to \cite{Bagirov2014} for a definition of subdifferential for locally Lipschitz functions. Subdifferentials and subgradients have been originally defined for non-differentiable convex functions and later extended to locallly Lipschitz continuous functions by Clarke \cite{Clarke1990}. Throughout this article, we assume Clarke's definition of subdifferential and subgradient.

We conclude this section with introducing critical points. A point $x \in \T$ is called \emph{critical} if $0 \in \partial{F(x)}$. 
Examples of critical points are the global minimizers of active component functions.
Figure~\ref{fig:ex_FrechetFunction} depicts an example of a Fr\'echet function and some of its critical points.

\begin{figure}[t]
\centering
\includegraphics[width=0.6\textwidth]{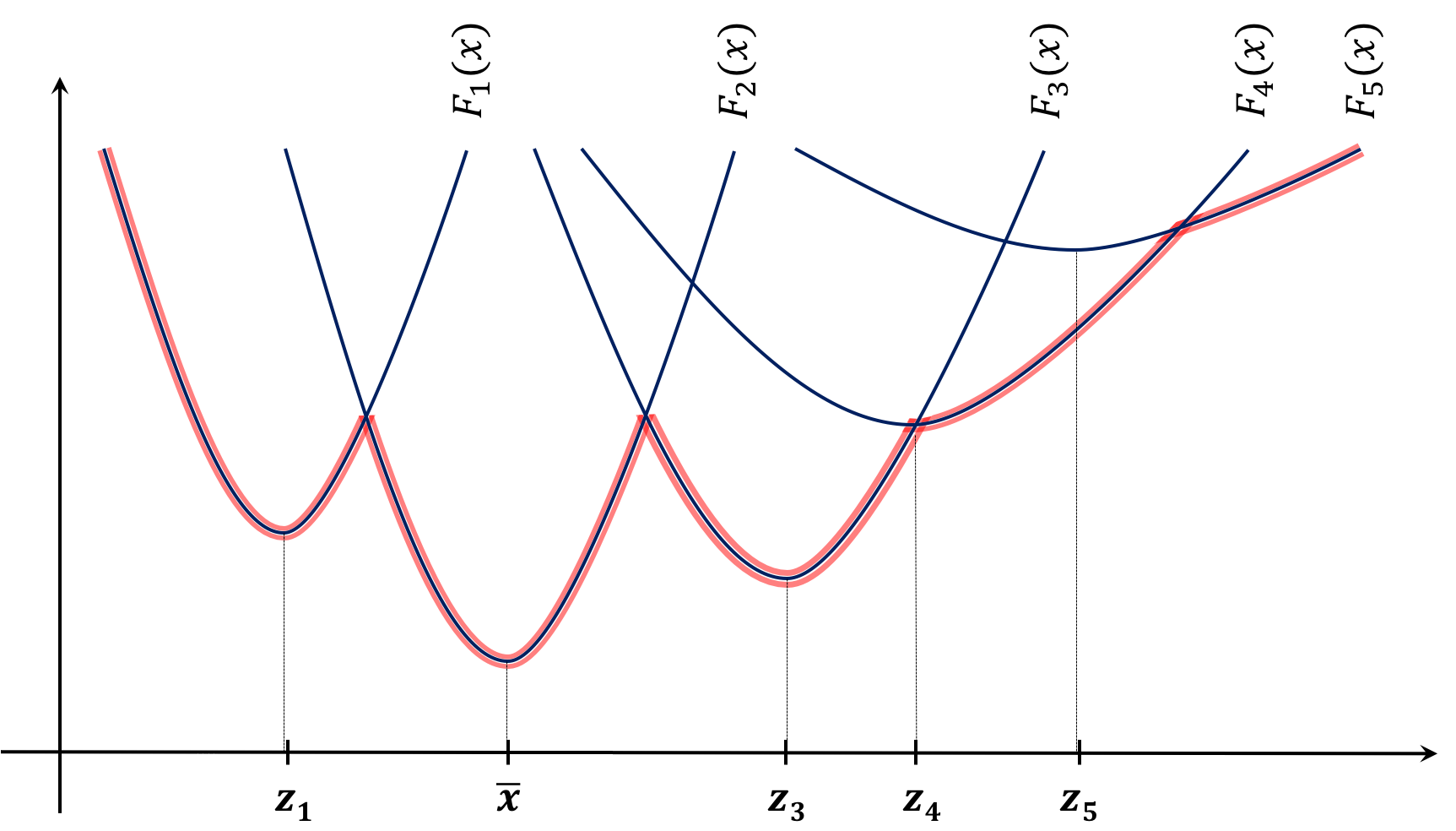}
\caption{Schematic depiction of Fr\'echet function $F$ of a sample $\S{X} \in \S{T}^N$ as pointwise minimum of the component functions $F_1, \dots, F_5$ indexed by integers rather than configurations from $\S{P}^N$. The component functions are shown by blue lines and the Fr\'echet function is depicted by the red line. The red line also shows the surface of the points at which the respective component functions are active. Though the component functions are convex and differentiable, the Fr\'echet function is neither convex nor differentiable but differentiable almost everywhere. The minimizer of component function $F_2$ is the unique global minimizer of $F$ and therefore the unique mean $\bar{x}$ of sample $\S{X}$. The minimizers $z_1$ and $z_3$ of the component functions $F_1$ and $F_3$, resp., are local minimizers of $F$. 
The minimizer $z_4$ of component function $F_4$ is a non-differentiable critical point of $F$. 
Finally, the minimizer $z_5$ of component function $F_5$ is neither a local minimizer nor a non-differentiable critical point of $F$, because $F_5$ is not active at $z_5$.}
\label{fig:ex_FrechetFunction}
\end{figure}

\subsection{Subgradients of the Fr\'echet Function}
This section aims at describing an arbitrary subgradient of $F$ for each point $x \in \T$.
As discussed in Section \ref{ss:localLipschitz} the gradient of an active component function $F_{\S{C}}$ at $x$ is a subgradient of $F$ at $x$. Since for each $x \in \T$ there is an active component function $F_{\S{C}} \in \S{A}_F(x)$, it is sufficient to specify gradients of component functions. For this, we introduce the notions of warping and valence matrix. 
\begin{definition}
Let $p \in \S{P}$ be a warping path. 
\begin{enumerate}
\item The \emph{warping matrix} of $p$ is a matrix $W \in \{0,1\}^{n \times n}$ with elements
\begin{equation*}
W_{i,j} = \begin{cases} 
1 & (i,j) \in p \\ 
0 & \text{otherwise} 
\end{cases}.
\end{equation*}
\item The \emph{valence matrix} of $p$ is the diagonal matrix $V \in \N^{n \times n}$ with integer elements
\begin{align*}
V_{i,i} = \sum_{j=1}^n W_{i,j}.
\end{align*}
\end{enumerate}
\end{definition}
Figure \ref{fig:valence} provides an example of a warping and valence matrix. The warping matrix is a matrix representation of its corresponding warping path. The valence matrix is a diagonal matrix, whose elements count how often an element of the first time series is aligned to an element of the second one.

The next result describes the gradients of a component function. 
\begin{proposition}\label{prop:DF_C} Let $\S{X} = \args{x^{(1)},\dots,x^{(N)}} \in \S{T}^N$ be a sample of time series and let $\S{C} \in \S{P}^N$ be a configuration of warping paths. The gradient of component function $F_{\S{C}}$ at $x \in \S{T}$ is of the form
\begin{align}\label{eq:prop:DF_C}
\nabla F_{\S{C}}(x) = \frac{2}{N} \,\sum_{k=1}^N \Big(V^{(k)}x - W^{(k)}x^{(k)}\Big),
\end{align}
where $V^{(k)}$ and $W^{(k)}$ are the valence and warping matrix of warping path $p^{(k)} \in \S{C}$ associated with time series $x^{(k)}\in \S{X}$. 
\end{proposition}

\begin{figure}[t]
\centering
\includegraphics[width=0.9\textwidth]{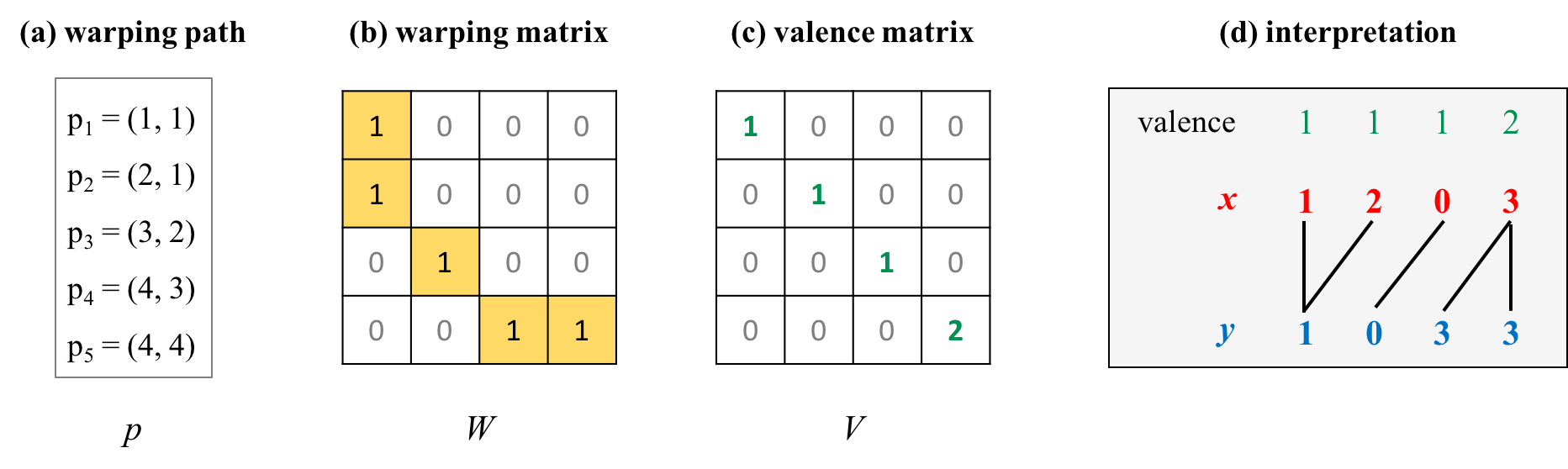}
\caption{Illustration of warping and valence matrix. Box (a) shows the warping path $p$ of Figure \ref{fig:wpath}. Box (b) shows the warping matrix $W$ of warping path $p$. The points $p_l = (i_l, j_l)$ of warping path $p$ determine the ones in the warping matrix $W$. Box (c) shows the valence matrix $V$ of warping path $p$. The matrix $V$ is a diagonal matrix, whose elements $V_{i,i}$ are the row sums of $W$. Box (d) interprets the valence matrix $V$. The valence $V_{i,i}$ of element $x_i$ of time series $x$ is the number of elements in time series $y$ that are aligned to $x_i$ by warping path $p$. Thus, the valence $V_{i,i}$ is the number of black lines emanating from element $x_i$.}
\label{fig:valence}
\end{figure}

\subsection{Necessary and Sufficient Conditions of Optimality}

The sample mean problem is posed as an unconstrained optimization problem. The classical mathematical approach to an optimization problem studies the four basic questions: (i) the formulation of necessary conditions, (ii) the formulation of sufficient conditions, (iii) the question of the existence of solutions, and (iv) the question of the uniqueness of a solution.
Existence of a sample mean has been proved \cite{Jain2016b} and non-uniqueness follows by constructing examples. In this section, we answer the remaining two questions (i) and (ii). 

The next theorem presents the necessary conditions of optimality in terms of warping and valence matrices. For an illustration of the theorem, we refer to Figure \ref{fig:mean}.
\begin{theorem} \label{theorem:form}
Let $F$ be the Fr\'echet function of sample $\S{X} = \args{x^{(1)},\dots,x^{(N)}} \in \S{T}^N$. 
If $z \in \S{T}$ is a local minimizer of $F$, then there is a configuration $\S{C} \in \S{P}^N$ such that the following conditions are satisfied:
\begin{description}
\item[(C1)] $F(z) = F_{\S{C}}(z)$.
\item[(C2)] We have
\begin{align}\label{eq:theorem:form}
z = \argsS{\sum_{k = 1}^N V^{(k)}}{^{-1}} \args{\sum_{k = 1}^N W^{(k)} \,x^{(k)}},
\end{align}
where $V^{(k)}$ and $W^{(k)}$ are the valence and warping matrix of $p^{(k)}\in \S{C}$ for all $k \in [N]$. 
\end{description}
\end{theorem}

Since (C1) and (C2) are necessary but not sufficient conditions, there are time series that satisfy both conditions but are not local minimizers of the Fr\'echet function. Condition (C2) implies that $z$ is the unique minimizer of the component function $F_{\S{C}}$ (cf.~\ref{ss:proofOfMainThm}). As illustrated in Figure \ref{fig:ex_FrechetFunction}, three cases can occur: (i) $z$ is a local minimizer of $F$, (ii) $z$ is a non-differentiable critical point, and (iii) $z$ is neither a local minimizer nor a non-differentiable critical point. Condition (C1) eliminates solutions of the third case, which are minimizers of inactive component functions. Hence, solutions satisfying the necessary conditions of optimality include local minimizers and non-differentiable critical point. The second result of this section presents a sufficient condition of a local minimum.

\begin{figure}[t]
\centering
\includegraphics[width=0.8\textwidth]{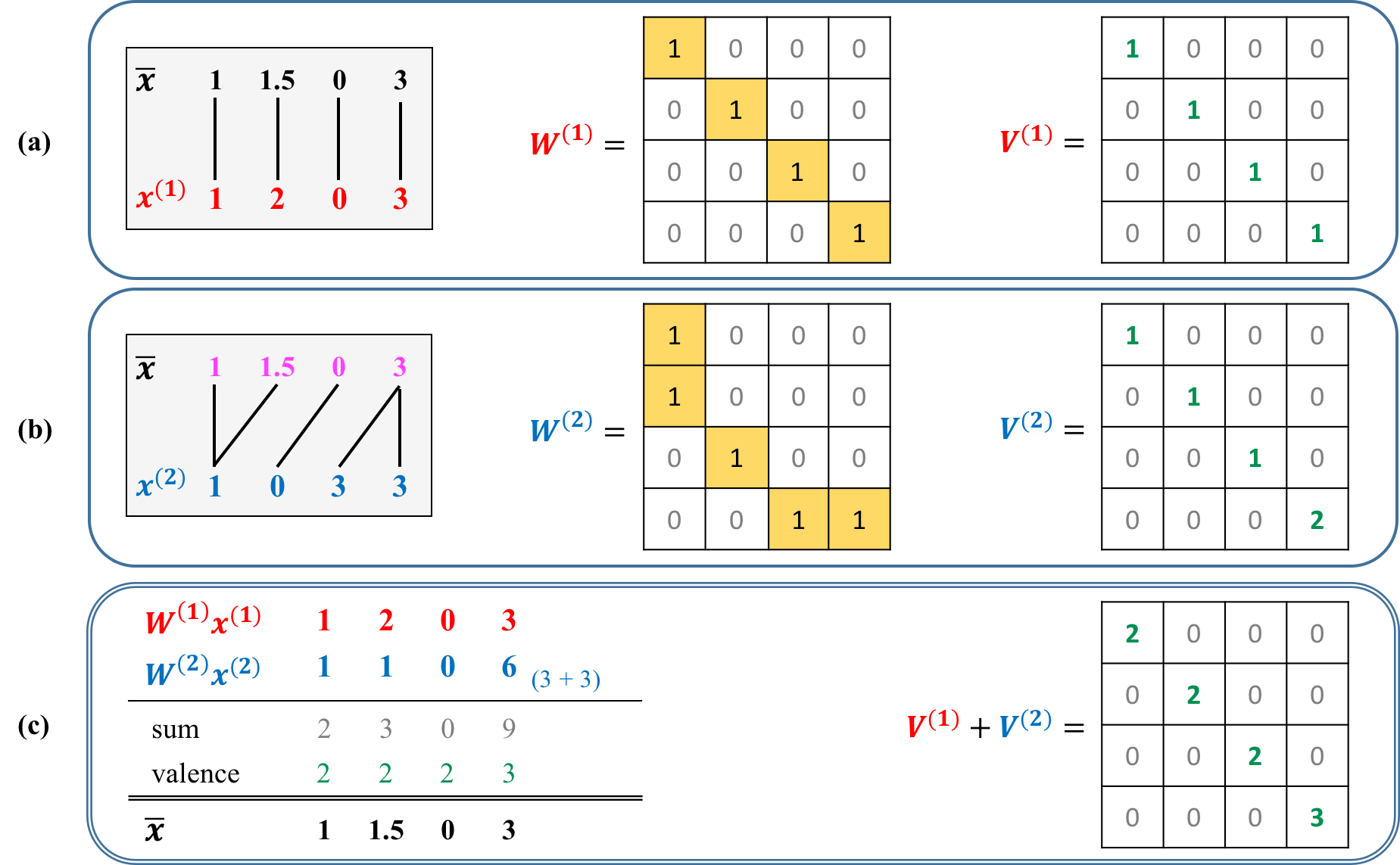}
\caption{Mean $\bar{x}$ of time series $x^{(1)}$ and $x^{(2)}$. The two boxes (a) and (b) show the warping matrix $W^{(k)}$ and the valence matrix $V^{(k)}$ of the optimal warping path between mean $\bar{x}$ and time series $x^{(k)}$ for $k \in \cbrace{1, 2}$. Box (c) shows computation of the mean $\bar{x}$ using the warping and valence matrices. The matrix on the right hand side shows the sum $V = V^{(1)} + V^{(2)}$ of both valence matrices. The first two lines on the left hand side show the results of the matrix multiplication $y^{(k)} = W^{(k)} \cdot x^{(k)}$. The third line shows the sum $y = y^{(1)} + y^{(2)}$. The fourth line shows the valences corresponding to the diagonal elements of $V$. The valences count how many elements of the sample time series are aligned to a given element of the mean. The last line shows the mean $\bar{x}$ obtained by element-wise division of the sum by the valences. In matrix notation, the mean is given by $\bar{x} = V^{-1}y$. }
\label{fig:mean}
\end{figure}

\begin{proposition}\label{prop:sufficient-condition}
Let $F$ be the Fr\'echet function of sample $\S{X}$ and let $z \in \S{T}$ be a time series. Suppose that there is a configuration $\S{C} \in \S{P}^N$ such that $z$ satisfies the necessary conditions (C1) and (C2). If $\S{C}$ is unique, then $z$ is a local minimizer. 
\end{proposition}

Uniqueness of configuration $\S{C}$ in Prop.~\ref{prop:sufficient-condition} is equivalent to the statement that the active set at $z$ is of the form $\S{A}_F(z) = \cbrace{F_{\S{C}}}$. Thus, the sufficient condition of Prop.~\ref{prop:sufficient-condition} is difficult to verify. One naive way to test the condition is to enumerate all optimal configurations at $z$. If there is exactly one optimal configuration, then $z$ is a local minimizer.

We conclude this section with a discussion on the form of a mean. Any sample mean is of the form of Eq.~\eqref{eq:theorem:form}, because global minimizers are also local minimizers. As expected, a sample mean in DTW spaces generalizes the arithmetic mean in Euclidean spaces. The term in Eq.~\eqref{eq:theorem:form} is a normalized sum of (transformed) sample time series. Intuitively speaking, the difference to the arithmetic mean is that the sample time series are aligned by warping transformations and the normalization factor becomes a diagonal matrix (inverted sum of the valence matrices) which normalizes each element of the time series separately.
Indeed, if every warping matrix is the identity matrix, then Eq.~\eqref{eq:theorem:form} is the sample mean of $\S{X}$ in the Euclidean sense.

\section{Mean Algorithms}\label{sec:algorithms}

In the previous section, we have shown that Fr\'echet functions are locally Lipschitz continuous. For such functions, the field of nonsmooth analysis has developed several optimization methods \cite{Bagirov2014}. One of the most simplest nonsmooth optimization techniques are subgradient methods originally developed for minimization of convex functions \cite{Shor1985}. This section adopts three (generalized) subgradient methods for minimizing the Fr\'echet function: (1) the subgradient mean algorithm, (2) the majorize-minimize mean algorithm, and (3) the stochastic version of the subgradient mean algorithm.
We provide MATLAB and Java implementations of the last two algorithms at \cite{Schultz2016Code}.

\subsection{The Subgradient Mean Algorithm}

Subgradient methods for minimizing a locally Lipschitz continuous function $f(x)$ look similar to gradient-descent algorithms. The update rule of the standard subgradient method is of the form
\begin{align}\label{eq:SG_update_rule}
z^{(t+1)} = z^{(t)} - \eta^{(t)} \cdot g^{(t)},
\end{align}
where $z^{(t)}$ is the solution at iteration $t$, $\eta^{(t)}$ is the $t$-th step size, and $g^{(t)} \in \partial f(z^{(t)})$ is an arbitrary subgradient of $f$ at $z^{(t)}$. Since the standard subgradient method is not a descent method, it is common to keep track of the best solution $z_*^{(t)}$ found so far, that is 
\[
f\!\args{z_*^{(t)}} = \min \cbrace{f\!\args{z_*^{(t-1)}}, f\!\args{z^{(t)}}}.
\]
with $z_*^{(0)} = z^{(0)}$. We apply this idea to the sample mean problem by using the gradient of an arbitrary active component function (cf.~Prop.~\ref{prop:DF_C}) as subgradient of the Fr\'echet function $F$ at the current solution $z^{(t)}$.

\begin{footnotesize}
\begin{algorithm}[t]
\caption{\footnotesize Subgradient Mean Algorithm}\label{alg:SG}
\begin{algorithmic}[1]
\footnotesize 
\Procedure{SG}{$x^{(1)},\dots,x^{(N)}$}
\State initialize solution $z \in \S{T}$
\State initialize best solution $z_* = z$
\Repeat
\OptParFor{all $k \in [N]$}
\State{compute optimal warping path $p^{(k)} \in \S{P}_*\!\args{z, x^{(k)}}$ between $z$ and $x^{(k)}$}
\State{derive valence matrix $V^{(k)}$ and warping matrix $W^{(k)}$ of $p^{(k)}$}
\EndOptParFor
\State{update solution $z$ according to the rule
\[
z \gets z - \eta \, \dfrac{2}{N} \sum_{k=1}^N \args{V^{(k)}z - W^{(k)}x^{(k)}}
\]
\State{update best solution $z_*$ such that $F(z_*) = \min\cbrace{F(z_*), F(z)}$}
\State{adjust step size $\eta$}
}
\Until{termination} 
\State{\Return $z_*$ as approximation of a mean}
\EndProcedure
\end{algorithmic}
\end{algorithm}
\end{footnotesize}

Algorithm \ref{alg:SG} outlines the basic procedure of the subgradient mean (SG) algorithm for minimizing the Fr\'echet function $F(x)$ of a sample $\S{X} = \args{x^{(1)},\dots,x^{(N)}} \in \S{T}^N$ of time series. 
The SG algorithm consists of the following steps:

\smallskip
\noindent 
1.\ \emph{Initialize}:
The performance of the SG algorithm depends on the choice of initial solution. Efficient and simple strategies to initialize the solution in \texttt{line 2} are randomly selecting a sample time series $x^{(k)}$ or generating a random time series. Several authors independently suggested to select a medoid time series from $\S{X}$  \cite{Hautamaki2008,Oates1999,Petitjean2016}, which requires $\Oh(N^2)$ computations of the DTW distance. For large sample size $N$, selecting a medoid of $\S{X}$ is computationally infeasible. An efficient alternative is to select a medoid of a randomly chosen sub-sample of $\S{X}$. 

\smallskip
\noindent 
2.\ \emph{Align sample time series}:
In \texttt{line 5-8}, the SG algorithm cycles through all time series of the sample $\S{X}$ and determines a configuration of optimal warping paths $p^{(k)}$ between the current solution $z$ and the sample time series $x^{(k)}$. Next, the valence and warping matrix of the optimal warping paths $p^{(k)}$ are derived. This step implicitly determines an active component function $F_{\S{C}}$ at the current solution $z$. 
It is the computationally most demanding part of the algorithm, but can be easily executed in $N$ parallel threads, where each thread executes \texttt{line 6-7} for a different $k \in [N]$. 

\smallskip
\noindent 
3.\ \emph{Update solution}: \texttt{Line 9} implements the update rule of the standard subgradient method corresponding to Eq.~\eqref{eq:SG_update_rule}, where the subgradient corresponds to the gradient of the active component function $F_\S{C}$.

\smallskip
\noindent 
4.\ \emph{Update best solution}:
\texttt{Line 10} keeps track of the best solution $z_*$ found so far. Updating the best solution evaluates the Fr\'echet function $F(z)$ at the new solution $z$, which requires $N$ computations of the DTW distance. To keep the computational effort low, the resulting optimal warping paths are used in the alignment step 2.~of the next iteration. Consequently, this step consumes no substantial additional computational resources, except at the last iteration.

\smallskip
\noindent 
5.\ \emph{Adjust step size}: In \texttt{line 11}, the step size is adjusted according to some schedule. There are several schedules for adapting the step size including the special case of a constant step size. The choice of a schedule for adjusting the step size affects the convergence behavior of the algorithm. 

\smallskip
\noindent 
6.\ \emph{Terminate}: A good termination criterion is important, because a premature termination by a weak criterion may result in a useless solution, whereas a too severe criterion may result in an iteration process that is computationally infeasible. We suggest the following two termination criteria: (T1) Terminate when a maximum number of iterations has been exceeded; and (T2) terminate when no improvement is observed after some iterations. Improvement in (T2) is measured by the evaluations of the Fr\'echet function in \texttt{line 10}. 

\medskip

We conclude this section by placing the SG algorithm into the context of existing mean algorithms. As the DBA algorithm, the SG algorithm falls into the class of asymmetric-batch methods (cf.~Section \ref{sec:related-work}). The difference between existing asymmetric-batch methods and the SG algorithm is the way with which the sample time series are weighted when synthesized to an average. Existing asymmetric-batch approaches use a fixed weighting scheme, whereas the SG algorithm admits a more general and flexible weighting scheme based on adaptive step sizes. In Section \ref{subsec:MM} we show that this flexibility includes the DBA algorithm. This means, DBA is a special case of the SG algorithm.

\subsection{The Majorize-Minimize Mean Algorithm}\label{subsec:MM}

\begin{footnotesize}
\begin{algorithm}[t]
\caption{\footnotesize Majorize-Minimize Mean Algorithm}\label{alg:MM}
\begin{algorithmic}[1]
\footnotesize 
\Procedure{MM}{$x^{(1)},\dots,x^{(N)}$}
\State{initialize solution $z \in \S{T}$}
\Repeat
\State{\texttt{//*** Majorization ***********************************************************//}}
\OptParFor{all $k \in [N]$}
\State{compute optimal warping path $p^{(k)} \in \S{P}_*\!\args{z, x^{(k)}}$ between $z$ and $x^{(k)}$}
\State{derive valence matrix $V^{(k)}$ and warping matrix $W^{(k)}$ of $p^{(k)}$}
\EndOptParFor
\State{\texttt{//*** Minimization **********************************************************//}}
\State{update solution $z$ according to the rule
\[
z \gets \argsS{\sum_{k=1}^N V^{(k)}}{^{-1}} \args{\sum_{k=1}^N W^{(k)}x^{(k)}}
\]
}
\Until{$F(z)$ is equal to the Fr\'echet variation of the previous iteration} 
\State{\Return $z$}
\EndProcedure
\end{algorithmic}
\end{algorithm}
\end{footnotesize}

The second mean algorithm exploits the necessary conditions of optimality stated in Theorem \ref{theorem:form} and belongs to the class of majorize-minimize algorithms \cite{Hunter2004}. This class includes the EM algorithm as special case and provides access to general convergence results \cite{Zangwill1969}. 
Algorithm \ref{alg:MM} outlines the basic procedure of the majorize-minimize mean (MM) algorithm. After initialization, the MM algorithm repeatedly alternates between majorization and minimization until convergence (cf.~Figure \ref{fig:ex_FrechetFunction_MM}):
\paragraph*{Majorization.}
A real-valued function $G: \S{T} \rightarrow \R$ majorizes the Fr\'echet function $F$ at time series $z \in \S{T}$ if $F(z) = G(z)$ and $F(x) \leq G(x)$ for all time series $x \in \S{T}$. Typically, one chooses a majorizing function that is much easier to minimize than the original function. For example, any active component function $F_{\S{C}} \in \S{A}_F(z)$ majorizes the Fr\'echet function $F$ at $z$. \texttt{Line 5-8} implicitly determine such a component function $F_{\S{C}}$ by constructing a configuration $\S{C}$ of optimal warping paths $p^{(k)} \in \S{P}_*\!\args{z, x^{(k)}}$ between the current solution $z$ and sample time series $x^{(k)}$.

\paragraph*{Minimization.} The minimization step in \texttt{line 10} computes the minimum of the majorizing function $F_{\S{C}}$ as next solution of the iteration process. Since an active component function is convex and differentiable, its unique minimizer is easy to compute by setting the gradient $\nabla F_{\S{C}}(z)$ given in Eq.~\eqref{eq:prop:DF_C} to zero and solving for $z$. The resulting minimizer takes the form of the necessary condition (C2) stated in Theorem \ref{theorem:form}.

\begin{figure}[t]
\centering
\includegraphics[width=0.6\textwidth]{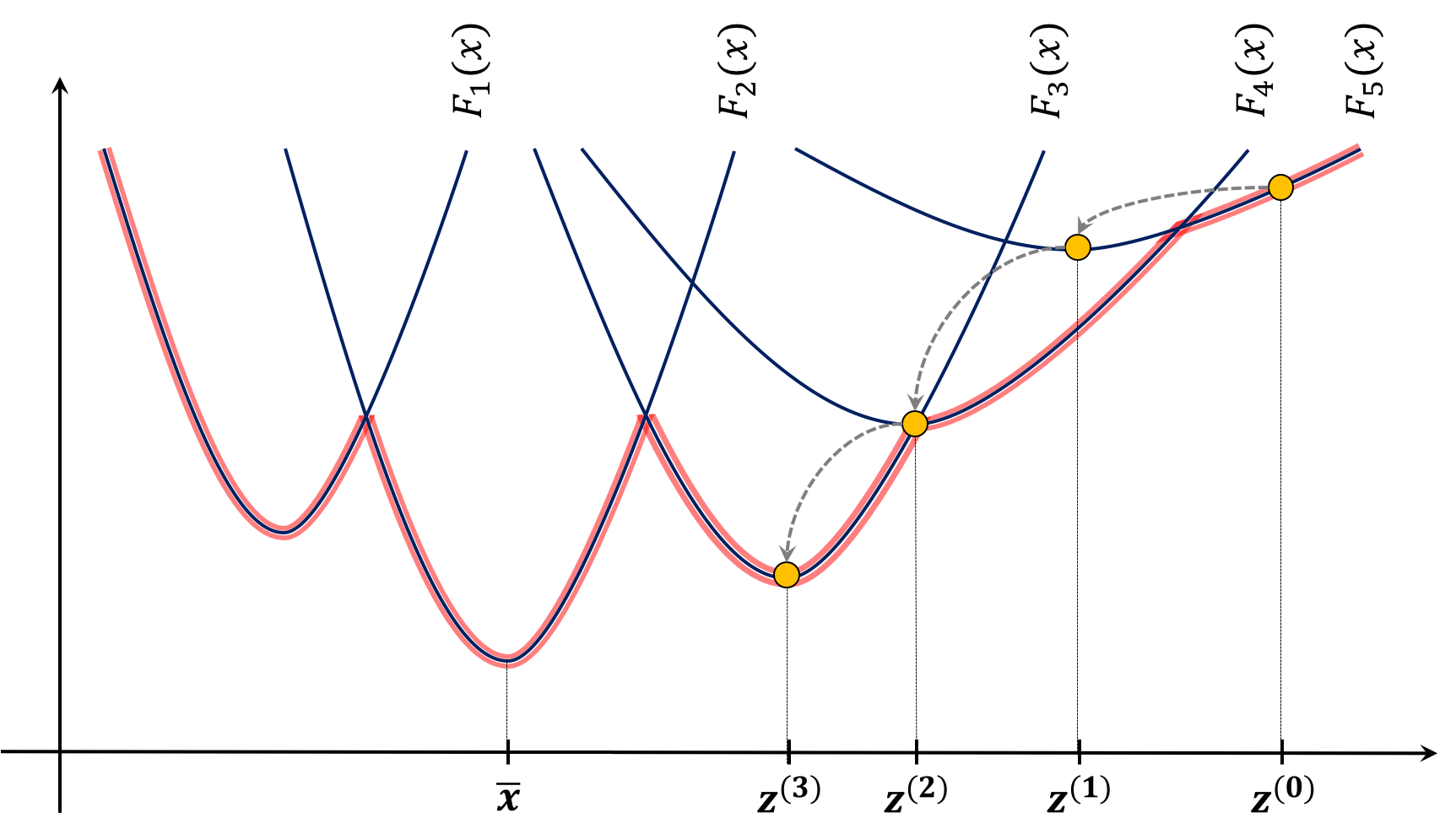}
\caption{Illustration of the MM algorithm. The algorithm starts with the initial solution $z^{(0)}$. At each iteration $t$, majorization determines a configuration $\S{C}^{(t)} \in \S{P}^N $ of optimal warping paths between the current solution $z^{(t)}$ and the sample time series $x^{(k)}$. The configuration $\S{C}^{(t)}$ determines the majorizing component function $F_{\S{C}^{(t)}}$, which is active at the current solution $z^{(t)}$. Then the minimization-step determines the unique minimum $x_*$ of the component function $F_{\S{C}^{(t)}}$. The iteration continues with $z^{(t+1)} = x_*$ as updated solution.}
\label{fig:ex_FrechetFunction_MM}
\end{figure}

\bigskip

Majorize-Minimize algorithms are descent methods by construction. Hence, if $(z^{(t)})$ is a sequence of points updated at iteration $t$ of the MM algorithm, then the sequence $(F(z^{(t)}))$ of Fr\'echet variations at the updated points $z^{(t)}$ is monotonously decreasing. The next theorem shows finite convergence of the MM algorithm to a solution satisfying necessary conditions of optimality. 
\begin{theorem}\label{theorem:convergence_of_MM}
The MM algorithm terminates after a finite number of iterations at a solution satisfying the necessary conditions of optimality (C1) and (C2).
\end{theorem}

One way to prove Theorem \ref{theorem:convergence_of_MM} is to invoke Zangwill's convergence theorem. 
Checking the assumptions of Zangwill's convergence theorem is often as difficult as finding a direct proof \cite{Bonnans2006}. Here, we present a direct proof of Theorem \ref{theorem:convergence_of_MM}, which is a stronger convergence statement than we would obtain by invoking Zangwill's convergence theorem. Theorem \ref{theorem:convergence_of_MM} yields the termination criterion used in Algorithm \ref{alg:MM}:

\begin{corollary}
Let $(z^{(1)}, z^{(2)}, \dots)$ be a sequence of updated points generated by the MM algorithm. If $F(z^{(t)}) = F(z^{(t-1)})$ for some iteration $t$, then $z^{(t)}$ satisfies the necessary conditions of optimality (C1) and (C2) stated in Theorem \ref{theorem:form}. Moreover, if the majorization step behaves deterministic in terms of the choice of optimal warping paths, then $z^{(t)}$ cannot be improved by the MM algorithm.
\end{corollary}

We conclude this section with placing the MM algorithm into broader context. First, the MM algorithm is equivalent to the DBA algorithm proposed by Petitjean et al.~\cite{Petitjean2011}. The valence and warping matrix in \texttt{line 7} contain all the information gathered in the association tables of the DBA algorithm. The update rules of both algorithms are identical but expressed in different ways. The MM algorithm is preferred in programming languages that benefit from matrix and vector operations instead of loops. In addition, the MM algorithm is easily parallelizable in the same manner as the SG algorithm. We want to note that the descent property of DBA has been proved in \cite{Petitjean2016}. Now, in context of majorize-minimize algorithms the descent property becomes self-evident. 

Second, the mean algorithms proposed by Soheily-Khah et al.~\cite{Soheily-Khah2016} use the same majorize-minimize optimization scheme as the MM algorithm for weighted and kernelized DTW distances. Their mean algorithm for the weighted DTW distance can be directly obtained by adapting the update rule in Algorithm \ref{alg:MM} corresponding to the subgradient of the Fr\'echet function with weighted DTW distance. For the kernelized DTW distance we have a maximization problem. In this case they use minorize-maximize optimization. Once an active component function is selected in the minorization step, they perform ordinary gradient acent on the selected component function in the maximization step. Therefore, the mean algorithms in \cite{Soheily-Khah2016} generalize DBA to weighted and kernelized DTW distances.

Third, the MM algorithm is closely related to the corresponding MM algorithms for computing a mean of graphs \cite{Jain2016a} and a mean partition in consensus clustering \cite{Dimitriadou2002}. All three algorithms have in common that the underlying Fr\'echet function is a pointwise minimum of convex differentiable component functions.

Fourth, the MM algorithm can be regarded as a special variant of a subgradient algorithm. By using the per-coordinate step size
\begin{align}\label{eq:eta-matrix}
\eta = \argsS{\frac{2}{N}\sum_{k=1}^N V^{(k)}}{^{-1}},
\end{align}
the SG Algorithm \ref {alg:SG} reduces to the MM algorithm. Observe that the step size in the standard SG algorithm is a positive-valued scalar. In contrast, the step size $\eta$ given in Eq.~\eqref{eq:eta-matrix} is a diagonal matrix that provides individual step sizes for each element (coordinate) of a time series.

\subsection{The Stochastic Subgradient Mean Algorithm}

The SG algorithm computes a subgradient on the basis of the complete sample of time series. In contrast, the stochastic subgradient mean (SSG) algorithm updates the current solution on the basis of a single randomly selected time series $x^{(k)} \in \S{X}$. The $k$-th term of the Fr\'echet function is of the form
\[
F^{(k)}(x) = \dtw^2\argsS{x, x^{(k)}} = \;\min_{p \in \S{P}} \; C_p\args{x, x^{(k)}}.
\]
As an immediate consequence of Prop.~\ref{prop:DF_C}, the gradient of component function $F_p^{(k)}(x) = C_p\args{x, x^{(k)}}$ is given by
\begin{align*}
\nabla F_p^{(k)}(x) = 2 \Big(V^{(k)}x - W^{(k)}x^{(k)}\Big).
\end{align*}
Algorithm \ref{alg:SSG} outlines the basic procedure of the stochastic subgradient mean algorithm for minimizing the Fr\'echet function of a sample $\S{X} = \args{x^{(1)},\dots,x^{(N)}} \in \S{T}^N$ of time series. The individual steps of the SSG and SG algorithm are similar to a large extent, yet a few points deserve further explanation:

\smallskip
\noindent 
1.\ \emph{Online setting}: Algorithm \ref{alg:SSG} presents the SSG algorithm as an incremental version of the SG algorithm that cycles several times through a given sample of size $N$. The SSG algorithm can also be applied in an online setting, where a stationary source with limited memory generates a sequence of time series. In this case, the SSG algorithm receives the next sample time series, updates the current solution, and then discards the current sample time series. There is no updating of the best solution. We can use (T1) as termination criterion (see discussion of the SG algorithm). When viewed as a stochastic optimization problem, we refer to \cite{Ermoliev1998,Norkin1986} for convergence results. 

\smallskip
\noindent 
2.\ \emph{Iteration}: Here, the number of iterations refers to the number of update steps. The SG algorithm considers the entire sample in one iteration, whereas the SSG algorithms considers a single sample time series in one iteration. 

\smallskip
\noindent 
3.\ \emph{Update best solution}:
Recall that keeping track of the best solution $z_*$ in \texttt{line 9} evaluates the Fr\'echet function $F(z)$, which requires $N$ computations of the DTW distance. While the SG algorithm can use the resulting optimal warping paths for deriving the valence and warping matrices in the next iteration, the SSG algorithm can reuse only a single optimal warping path. Thus, one cycle through the entire sample by the SGG algorithm requires nearly twice as many DTW distance computations as the SG algorithm. To reduce the computational cost, we can either omit this step or apply it in regular intervals. 

\smallskip
\noindent 
4.\ \emph{Terminate}: The discussion about additional computational cost in item (3) carries over to termination criterion (T2). 

\begin{footnotesize}
\begin{algorithm}[t]
\caption{\footnotesize Stochastic Subgradient Method}\label{alg:SSG}
\begin{algorithmic}[1]
\footnotesize 
\Procedure{SSG}{$x^{(1)},\dots,x^{(N)}$}
\State initialize solution $z \in \S{T}$
\State initialize best solution $z_* = z$
\Repeat
\State{randomly select $k \in [N]$}
\State{compute optimal warping path $p^{(k)} \in \S{P}_*\!\args{z, x^{(k)}}$ between $z$ and $x^{(k)}$}
\State{derive valence matrix $V^{(k)}$ and warping matrix $W^{(k)}$ of $p^{(k)}$}
\State{update solution $z$ according to the rule
\[
z \gets z - \eta \args{V^{(k)} z - W^{(k)}x^{(k)}}
\]
\State{update best solution $z_*$ such that $F(z_*) = \min\cbrace{F(z_*), F(z)}$}
\State{adjust step size $\eta$}
}
\Until{termination} 
\State{\Return $z$}
\EndProcedure
\end{algorithmic}
\end{algorithm}
\end{footnotesize}

\medskip

We conclude this section by classifying the SSG algorithm into the asymmetric-symmetric and batch-incremental dimension as described in Section \ref{sec:related-work}. The SSG algorithm implements incremental averaging by following a similar scheme as the symmetric-incremental method NLAAF2 \cite{Gupta1996}. In contrast to NLAAF2, the type of average in SGG is asymmetric. Thus, SSG is the first proposed mean algorithm belonging to the class of asymmetric-incremental methods.

\section{Experiments}\label{sec:experiments}

The goal of this section is to assess the performance of the SSG algorithm in comparison with the MM (DBA) algorithm.

\subsection{General Performance Comparison}\label{sec:exp01}

The first series of experiments compares the performance of different variants of SSG and MM on selected UCR benchmark datasets.

\subsubsection{Data} 
In this experiment, the $24$ datasets of the UCR Time Series Classification Archive \cite{Chen2015} shown in Table \ref{tab:data} were selected. Time series of the same dataset have identical length. The training and test set of the original datasets were merged to a single set. 
 
\subsubsection{Algorithms}
The following five variants of the SSG and MM algorithm were considered:
\begin{center}
\begin{tabular}{l@{\quad}l@{\quad}r}
\hline
Notation & Algorithm & Epochs\\
\hline
\\[-2ex]
SSG-1 & stochastic subgradient mean algorithm & 1\\
SSG-e & stochastic subgradient mean algorithm & $e \in [50]$\\
SSG-50 & stochastic subgradient mean algorithm & 50\\
MM-1 & majorize-minimize mean algorithm & 1\\
MM-50 & majorize-minimize mean algorithm & 50\\
\hline
\end{tabular}
\end{center}
One epoch is a cycle through the whole dataset. SSG-1 and MM-1 terminate after the first epoch and SSG-50 after $50$ epochs. MM-50 terminates as described in Algorithm \ref{alg:MM}, but at the latest after 50 epochs. Finally, SSG-e terminates after the same number of epochs as MM-50.

The step size of the SSG algorithms were adjusted according to the following schedule: 
\begin{align*}
\eta^{(t)} = \begin{cases} 
\displaystyle \eta^{(t-1)} - \args{\eta_0 - \eta_1}/N & 1 \leq t \leq N \\[3ex]
\displaystyle \eta_1 & t > N
\end{cases},
\end{align*}
where $\eta_0 = \eta^{(0)} = 0.05$ is the initial step size, $\eta_1 = 0.005$ is the final step size, $t$ is the number of iterations, and $N$ is the sample size. This schedule linearly decreases the step size from $\eta_0$ to $\eta_1$ during the first epoch and then remains constant at $\eta_1$ until termination. 

The solution quality of an algorithm is measured by means of Fr\'echet variation of the found solution. We denote the variation of algorithm $A$ at epoch $e$ by 
\[
V_A(e) = F\args{z^{(e)}_*},
\]
where $F$ is the Fr\'echet function and $z^{(e)}_*$ is the best solution found so far by algorithm $A$.

\begin{table}[t]
\footnotesize
\centering
\caption{UCR datasets and their characteristics. Shown are the name of the dataset, the length $n$ of the time series, the number $C$ of classes, and the sample size $N$.} 
\begin{tabular}{l || r | r | r}
\hline
\hline
Dataset 				& \multicolumn{1}{|c|}{$n$} & \multicolumn{1}{|c|}{$C$}& \multicolumn{1}{|c}{$N$}\\
\hline
50words 				& 270 & 50 & 905 \\ 
Adiac 				& 176 & 37 & 781 \\ 
Beef 				& 470 & 5 & 60 \\ 
CBF 					& 128 & 	 3 & 930 \\ 
ChlorineConcentration	& 166 & 3 & 4307 \\
Coffee 				& 286 & 2 & 56 \\ 
ECG200 				& 96 & 2 & 200 \\ 
ECG5000				& 140 & 	 5 &	 5000 \\
ElectricDevices 		& 96 & 7 & 16637 \\
FaceAll 				& 131 & 	14 & 2250 \\ 
FaceFour 				& 350 & 4 & 112 \\ 
FISH 				& 463 & 7 & 350 \\ 
\hline
\hline
\end{tabular}
\hspace{1em}
\begin{tabular}{l || r | r | r}
\hline
\hline
Dataset 				& \multicolumn{1}{|c|}{$n$} & \multicolumn{1}{|c|}{$C$}& \multicolumn{1}{|c}{$N$}\\
\hline
Gun Point 			& 150 & 	 2 & 200 \\ 
Lighting2 				& 637 & 	 2 & 121 \\ 
Lighting7 				& 319 & 	 7 & 	 143 \\ 
OliveOil 				& 570 & 	 4 & 60 \\ 
OSULeaf 				& 427 & 	 6 & 442 \\ 
PhalangesOutlinesCorrect& 80 & 2 & 2658 \\
SwedishLeaf 			& 128 & 	15 & 1125 \\ 
synthetic control 		& 60 &	 6 & 600 \\ 
Trace 				& 275 & 	 4 & 200 \\ 
Two Patterns 			& 128 & 	 4 & 5000 \\ 
wafer 				& 152 & 	 2 & 7164 \\ 
yoga					& 426 &	 2 &	 3300 \\
\hline
\hline
\end{tabular}
\label{tab:data}
\end{table}

\subsubsection{Experimental Protocol}

Given a dataset, an experiment was conducted according to the following procedure:
\begin{small}
\begin{enumerate}
\itemsep-0.2em
\item \textbf{given:} dataset $\S{X}$
\item \textbf{for} $30$ trials \textbf{do}
\item \qquad randomly select initial solution $z \in \S{X}$
\item \qquad run SSG for 50 epochs with initial solution $z$
\item \qquad record variation $V_{SSG}(e)$ for $e \in \cbrace{1, \dots, 50}$
\item \qquad run MM for 50 epochs with initial solution $z$
\item \qquad record variation $V_{MM}(e)$ for $e \in \cbrace{1, \dots, 50}$
\item \textbf{end for}
\end{enumerate}
\end{small}
Thus, in total, $720$ trials (24 datasets $\times$ 30 trials) were conducted. 

\subsubsection{Results}

Table \ref{tab:results} summarizes the results obtained by the five mean algorithms. The results show the average variations over the $30$ trials and their standard deviations. The best average variations vary from $0.03$ for OliveOil to $78.05$ for Lighting2 both obtained by SSG-50. These results suggest that the selected datasets cover a broad range of variability in the data. 

\begin{table}[t]
\scriptsize
\centering
\caption{Average variation (lower is better) and standard deviation over $30$ trials.}
\scriptsize
\begin{tabular}{l || r@{\:}r | r@{\:}r | r@{\:}r | r@{\:}r | r@{\:}r} 
\hline
\hline
Dataset & \multicolumn{2}{c|}{SSG-1} & \multicolumn{2}{c|}{SSG-e} & \multicolumn{2}{c|}{SSG-50} & \multicolumn{2}{c|}{MM-1} & \multicolumn{2}{c}{MM-50} \\
\hline 
&&&&&&&&\\[-2ex]
50words & $18.29$&$^{(\pm0.57)}$ & $17.71$&$^{(\pm0.53)}$ & $17.56$&$^{(\pm0.48)}$ & $34.67$&$^{(\pm5.28)}$ & $18.68$&$^{(\pm1.93)}$ \\
Adiac & $0.55$&$^{(\pm0.01)}$ & $0.53$&$^{(\pm0.00)}$ & $0.52$&$^{(\pm0.00)}$ & $0.8$&$^{(\pm0.05)}$ & $0.54$&$^{(\pm0.01)}$ \\
Beef & $27.48$&$^{(\pm4.51)}$ & $19.95$&$^{(\pm2.41)}$ & $15.9$&$^{(\pm1.16)}$ & $38.96$&$^{(\pm6.68)}$ & $16.86$&$^{(\pm2.24)}$ \\
CBF & $18.35$&$^{(\pm0.34)}$ & $18.06$&$^{(\pm0.30)}$ & $18.01$&$^{(\pm0.31)}$ & $23.68$&$^{(\pm1.67)}$ & $18.4$&$^{(\pm0.35)}$ \\
ChlorineConcentration & $14.79$&$^{(\pm0.35)}$ & $14.71$&$^{(\pm0.36)}$ & $14.71$&$^{(\pm0.35)}$ & $17.64$&$^{(\pm1.13)}$ & $15.13$&$^{(\pm0.52)}$ \\
Coffee & $0.77$&$^{(\pm0.03)}$ & $0.7$&$^{(\pm0.02)}$ & $0.67$&$^{(\pm0.01)}$ & $0.91$&$^{(\pm0.08)}$ & $0.68$&$^{(\pm0.01)}$ \\
ECG200 & $7.15$&$^{(\pm0.49)}$ & $7.15$&$^{(\pm0.49)}$ & $6.76$&$^{(\pm0.40)}$ & $8.81$&$^{(\pm0.78)}$ & $6.95$&$^{(\pm0.43)}$ \\
ECG5000 & $17.91$&$^{(\pm0.99)}$ & $17.79$&$^{(\pm0.97)}$ & $17.76$&$^{(\pm0.97)}$ & $26.92$&$^{(\pm4.22)}$ & $19.21$&$^{(\pm1.58)}$ \\
ElectricDevices & $42.74$&$^{(\pm0.27)}$ & $42.54$&$^{(\pm0.22)}$ & $42.48$&$^{(\pm0.22)}$ & $56.84$&$^{(\pm4.99)}$ & $43.54$&$^{(\pm0.49)}$ \\
FaceAll & $27.6$&$^{(\pm0.39)}$ & $27.44$&$^{(\pm0.36)}$ & $27.4$&$^{(\pm0.34)}$ & $33.11$&$^{(\pm1.65)}$ & $28.25$&$^{(\pm0.76)}$ \\
FaceFour & $37.69$&$^{(\pm1.68)}$ & $35.83$&$^{(\pm1.37)}$ & $35.18$&$^{(\pm1.37)}$ & $48.54$&$^{(\pm4.87)}$ & $36.43$&$^{(\pm1.53)}$ \\
FISH & $1.33$&$^{(\pm0.03)}$ & $1.29$&$^{(\pm0.03)}$ & $1.29$&$^{(\pm0.02)}$ & $1.67$&$^{(\pm0.16)}$ & $1.33$&$^{(\pm0.08)}$ \\
Gun\_Point & $2.72$&$^{(\pm0.34)}$ & $2.63$&$^{(\pm0.36)}$ & $2.41$&$^{(\pm0.29)}$ & $5.99$&$^{(\pm1.11)}$ & $2.4$&$^{(\pm0.20)}$ \\
Lighting2 & $88.17$&$^{(\pm3.71)}$ & $82.34$&$^{(\pm2.28)}$ & $78.05$&$^{(\pm1.45)}$ & $116.92$&$^{(\pm9.30)}$ & $79.95$&$^{(\pm2.41)}$ \\
Lighting7 & $52.73$&$^{(\pm1.97)}$ & $48.88$&$^{(\pm1.00)}$ & $48.59$&$^{(\pm0.94)}$ & $70.32$&$^{(\pm4.50)}$ & $49.51$&$^{(\pm1.20)}$ \\
OliveOil & $0.03$&$^{(\pm0.00)}$ & $0.03$&$^{(\pm0.00)}$ & $0.03$&$^{(\pm0.00)}$ & $0.03$&$^{(\pm0.00)}$ & $0.03$&$^{(\pm0.00)}$ \\
OSULeaf & $29.15$&$^{(\pm0.59)}$ & $29.15$&$^{(\pm0.59)}$ & $27.68$&$^{(\pm0.39)}$ & $51.72$&$^{(\pm8.67)}$ & $28.41$&$^{(\pm0.71)}$ \\
PhalangesOutlinesCorrect & $1.17$&$^{(\pm0.01)}$ & $1.16$&$^{(\pm0.01)}$ & $1.16$&$^{(\pm0.01)}$ & $1.48$&$^{(\pm0.10)}$ & $1.17$&$^{(\pm0.01)}$ \\
SwedishLeaf & $4.08$&$^{(\pm0.09)}$ & $4.03$&$^{(\pm0.09)}$ & $4.01$&$^{(\pm0.08)}$ & $5.1$&$^{(\pm1.08)}$ & $4.1$&$^{(\pm0.15)}$ \\
synthetic\_control & $21.79$&$^{(\pm0.19)}$ & $21.53$&$^{(\pm0.16)}$ & $21.31$&$^{(\pm0.16)}$ & $28.58$&$^{(\pm2.42)}$ & $21.72$&$^{(\pm0.16)}$ \\
Trace & $35.69$&$^{(\pm18.1)}$ & $28.75$&$^{(\pm19.7)}$ & $28.4$&$^{(\pm19.8)}$ & $72.47$&$^{(\pm33.6)}$ & $22.67$&$^{(\pm6.76)}$ \\
Two\_Patterns & $13.69$&$^{(\pm1.46)}$ & $13.56$&$^{(\pm1.41)}$ & $13.55$&$^{(\pm1.41)}$ & $26.5$&$^{(\pm2.41)}$ & $17.67$&$^{(\pm2.64)}$ \\
wafer & $24.41$&$^{(\pm1.93)}$ & $23.71$&$^{(\pm1.97)}$ & $23.54$&$^{(\pm1.62)}$ & $55.56$&$^{(\pm6.33)}$ & $29.03$&$^{(\pm3.89)}$ \\
yoga & $11.93$&$^{(\pm0.36)}$ & $11.7$&$^{(\pm0.31)}$ & $11.63$&$^{(\pm0.32)}$ & $29.28$&$^{(\pm6.34)}$ & $12.64$&$^{(\pm1.41)}$ \\
\hline
\hline
\end{tabular}
\label{tab:results}
\end{table}

\begin{figure}[t]
\centering
\includegraphics[width=0.49\textwidth]{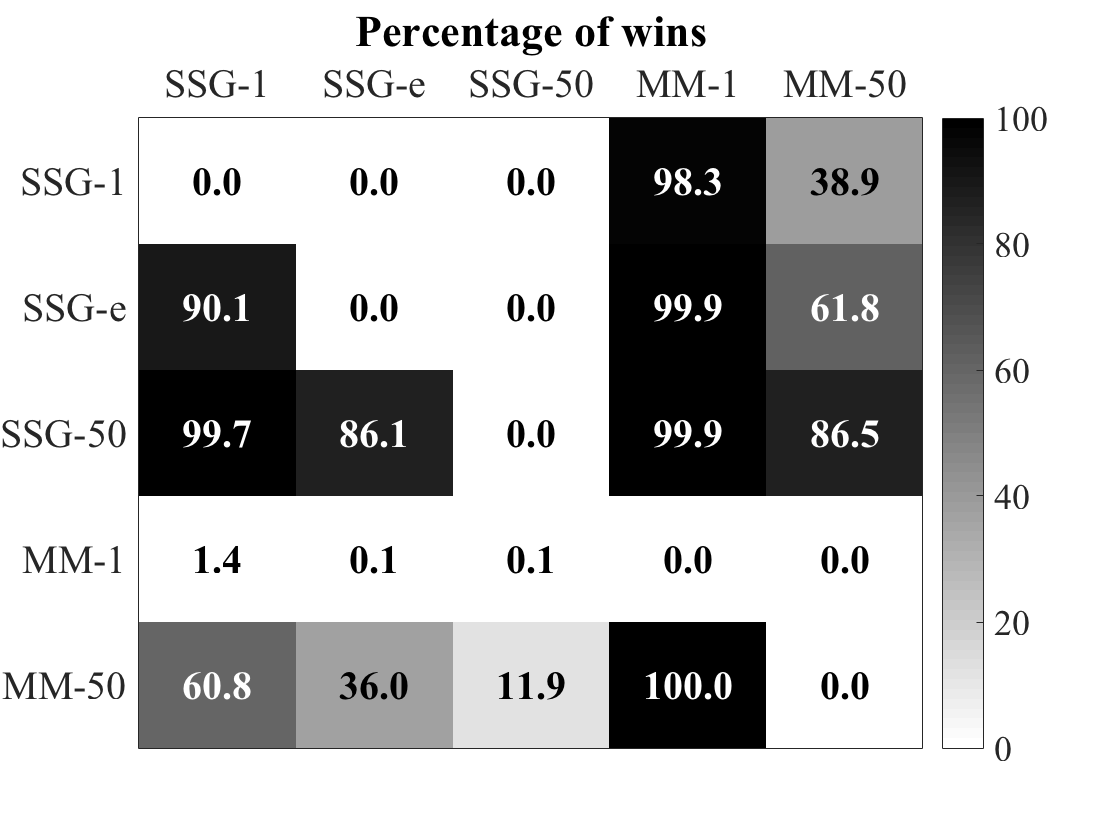}
\includegraphics[width=0.49\textwidth]{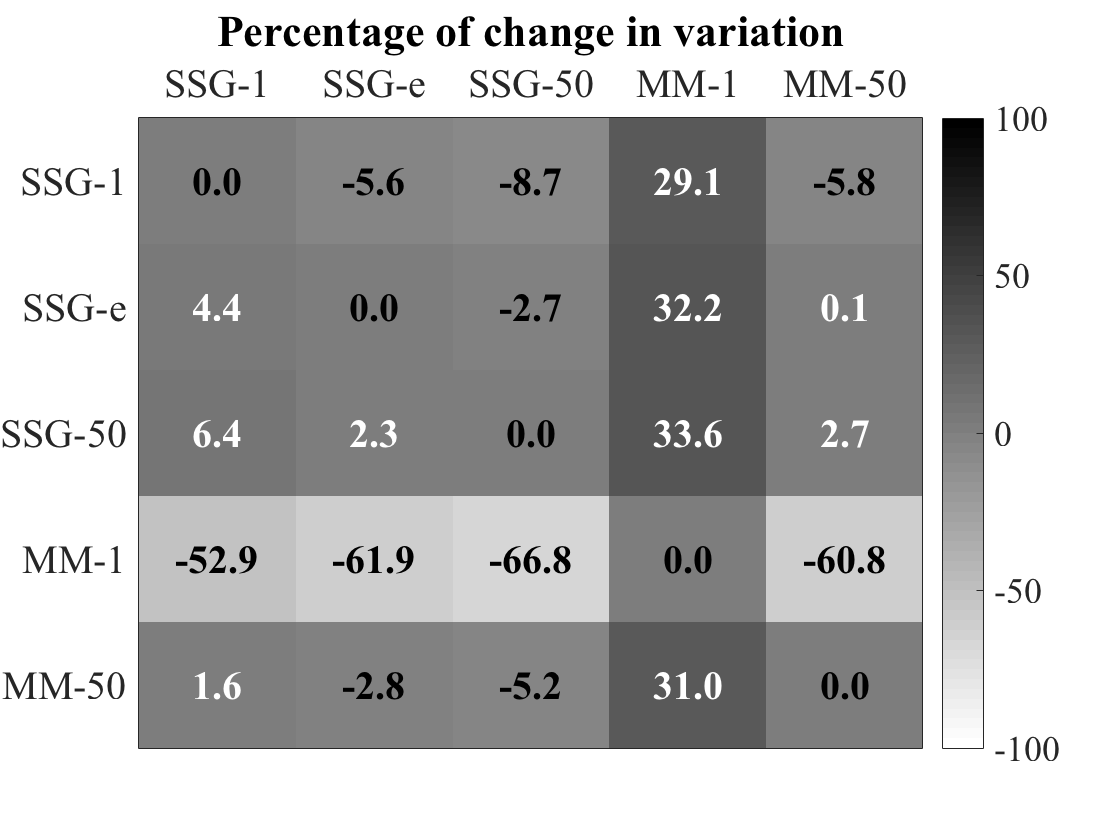}
\caption{Percentage of wins (left) (row against column) and average percentage of change (right) over $720$ trials.}
\label{fig:rnk}
\end{figure}

We compare the solution quality of the five mean algorithms. The left plot of Figure \ref{fig:rnk} shows the matrix $\args{p_{ij}^w}$ of percentage of wins. The percentage $p_{ij}^w$ of wins is the fraction of trials in which the algorithm in row $i$ found a solution with lower variation than the algorithm in column $j$. Note that 
\[
p_{ij}^{eq} = 100-p_{ij}^w-p_{ji}^w
\]
is the percentage of trials for which algorithms $i$ and $j$ found solutions with identical variation. The right plot of Figure \ref{fig:rnk} shows the average percentage $p_{ij}^{\,v}$ of change in variation over all $720$ trials. The percentage $p_{ij}^v$ of change in variation of a single trial is defined by 
\[
p_{ij}^v = 100\cdot \frac{V_j-V_i}{V_j},
\]
 where $V_i$ and $V_j$ are the variations of the solutions found by algorithms in row $i$ and column $j$, respectively. Positive (negative) values $p_{ij}^v$ mean that the algorithm in row $i$ won (lost) against the reference algorithm in column $j$. We made the following observations:

\paragraph*{1.} SSG-1 won in $98.3 \%$ of all trials against MM-1 with an average improvement of $29.1 \%$. This result shows that SSG-1 substantially outperformed MM-1 and suggests to prefer SSG over MM in scenarios, where computation time matters. 

\paragraph*{2.} After one iteration SSG-1 deviates merely $8.7 \%$ from the best solution found by SSG-50. In contrast MM-1 deviates $60.8 \%$ from the best solution found by MM-50. This indicates that SSG converges substantially faster than MM. We quantify this observation further in Section \ref{sss:results_on_runtime}.

\paragraph*{3.}
 SSG-e won in $61.8 \%$ of all trials against MM-50 with an average improvement of $0.1 \%$. 
 Given the same amount of computation time defined by termination of MM, then SSG and MM perform comparable, though the results by SSG are slightly better than those obtained by MM.
 
\paragraph*{4.} 
SSG-50 won in $86.5 \%$ of all trials against MM-50 with an average improvement of $2.7 \%$. This result indicates that the solutions found by SSG-50 further improved after termination of MM-50. 
Consequently, in situations where computation time is not an issue, SSG is more likely to return better solutions than MM.

\paragraph*{5.} 
One limitation of SSG is that its performance depends on the schedule with which the step size parameter is adapted. Finding an optimal schedule can take much longer than a run of DBA. To mitigate this limitation, we proposed a default schedule and applied it to all experiments. On the positive side, optimizing the step size schedule in a problem-dependent manner gives room for further improvement of the SSG algorithm. This advantage can be exploited in scenarios, where computation time is not so much a critical issue. 

\medskip
 
In summary, on average SSG finds solutions of similar quality faster than MM and is therefore better suited in situations, where sample size is large or many sample means need to be computed, such as in k-means for large datasets.

\subsection{Comparison of Performance as a Function of the Sample Size}

The goal of the second series of experiments is to assess the variation of the MM and SSG algorithms as a function of the sample size.

\subsubsection{Experimental Setup} 
The four largest datasets of the $24$ UCR datasets shown in Table \ref{tab:data} were used: ECG5000, ElectricDevices, Two\_Patterns, and wafer. 

\medskip

Given a dataset $\S{X}$ and sample sizes $\S{N} = \{N_1 , \dots , N_r\}$, $N_i \leq \abs{\S{X}}$, which can be taken from Figure \ref{fig:ss}, an experiment was conducted as follows:
\begin{small}
\begin{enumerate}
\itemsep-0.2em
\item \textbf{given:} dataset $\S{X}$, sample sizes $\S{N}$
\item \textbf{for} $30$ trials \textbf{do} 
\item \qquad \textbf{for} $N \in \S{N}$ \textbf{do}
\item \qquad\qquad randomly draw a subsample $\S{X}_N$ of $\S{X}$ of size $N$.
\item \qquad\qquad randomly select initial solution $z \in \S{X}_N$
\item \qquad\qquad run SSG on $\S{X}_N$ for 50 epochs with initial solution $z$
\item \qquad\qquad run MM on $\S{X}_N$ for 50 epochs with initial solution $z$
\item \qquad\qquad record variations of MM and SSG after each epoch
\item \qquad\textbf{end for}
\item \textbf{end for}
\end{enumerate}
\end{small}
The step size schedule of SSG was chosen as in Section \ref{sec:exp01}.

\begin{figure}[h!]
\centering
\includegraphics[width=0.9\textwidth]{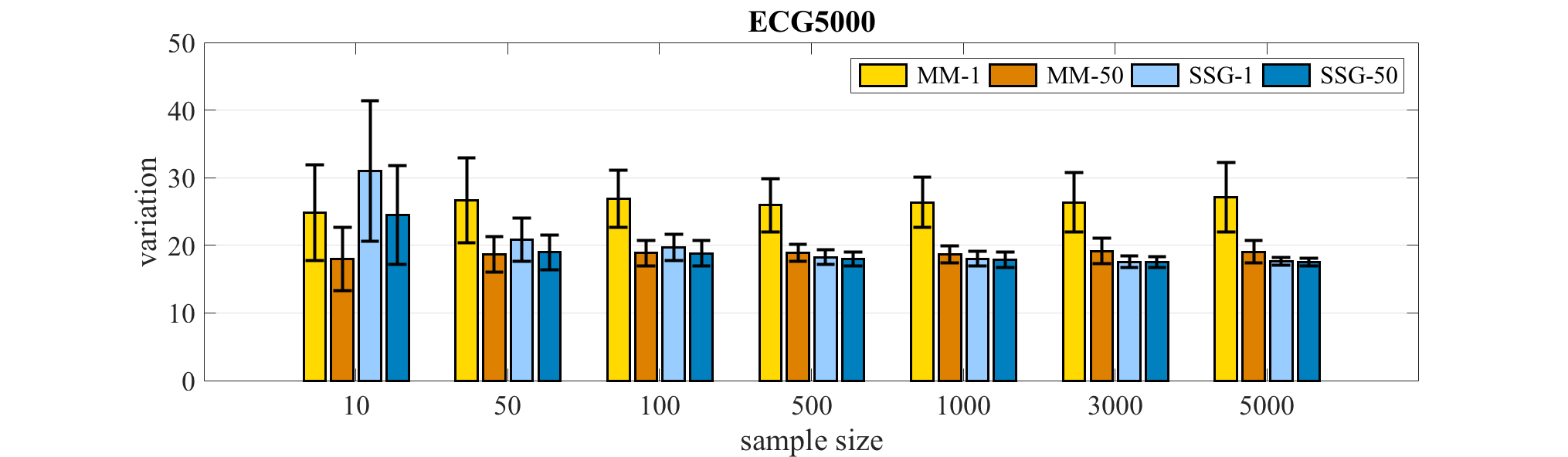}
\includegraphics[width=0.9\textwidth]{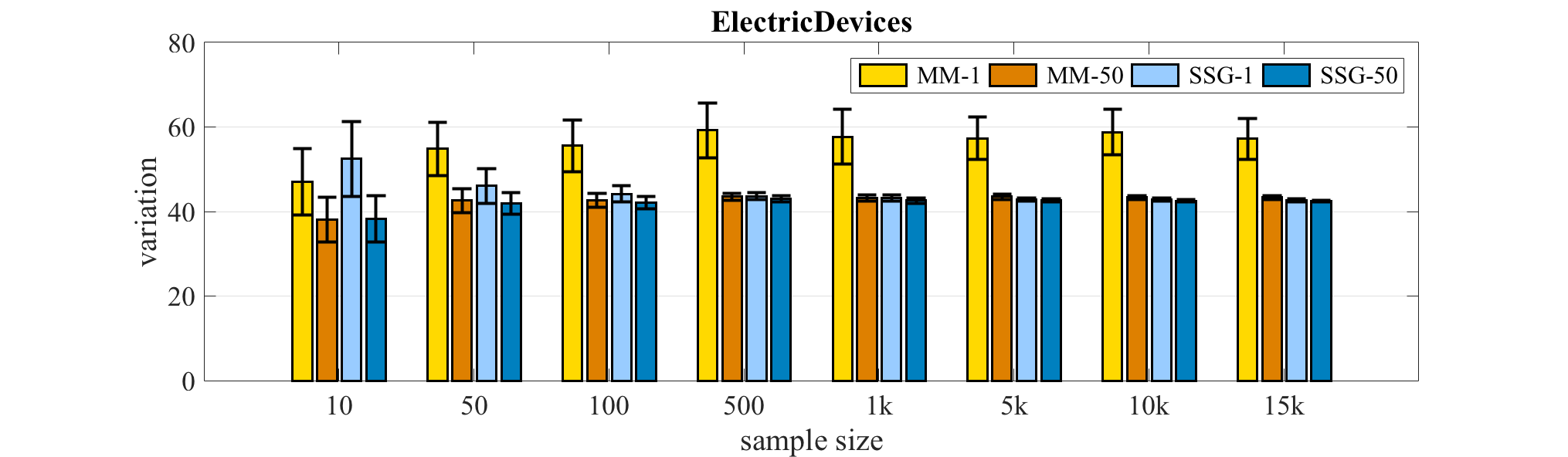}
\includegraphics[width=0.9\textwidth]{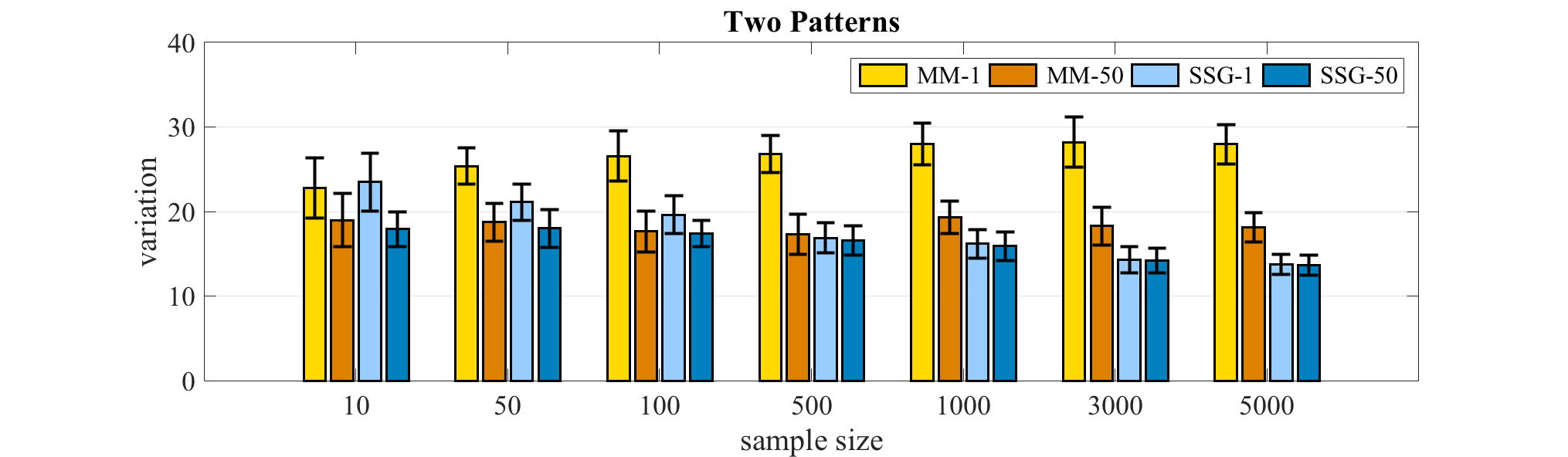}
\includegraphics[width=0.9\textwidth]{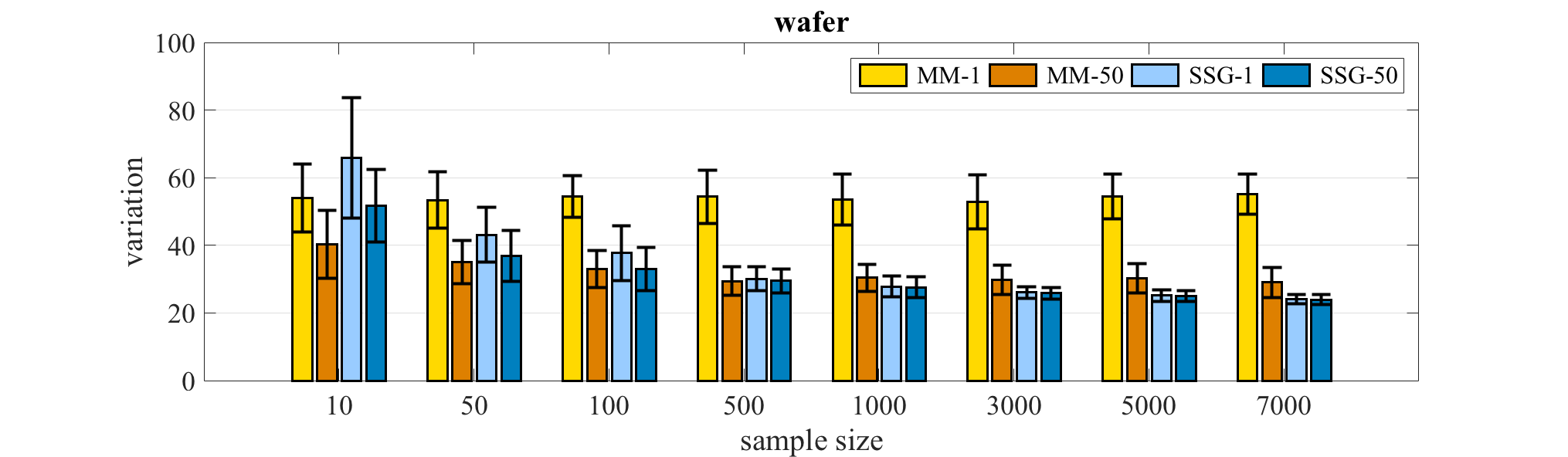}
\caption{Average variation (lower is better) and standard deviation of MM-1, MM-50, SSG-1 and SSG-50.}
\label{fig:ss}
\end{figure}

\subsubsection{Results on Variation}\label{sss:results_on_variation}

Figure \ref{fig:ss} shows the average variations and their standard deviations of the mean algorithms as a function of the sample size $N$. 
We made the following observations:

\paragraph*{1.} For large sample sizes ($N \geq 1000$), SSG-1 finds better solutions than MM-50 on average. Hence, for large sample sizes SSG finds better solutions in less time than MM on average.

\paragraph*{2.} The difference of average variations of SSG-50 and SSG-1 decreases with increasing sample size. For sample sizes $N \geq 500$ the solution qualities of SSG-1 and SSG-50 are comparable. This indicates that for sufficiently large sample sizes SSG may converge within the first epoch.

\paragraph*{3.} For increasing sample sizes, the difference of average variations of MM-1 and SSG-1 increases. Thus, increasing sample sizes have a positive influence on SSG-1 compared to MM-1. This indicates that the number of updates matters for solution quality. However, MM needs to cycle through the whole sample before making an update, whereas the number of updates made by SSG-1 increases with increasing sample size.

\paragraph*{4.} For all but the smallest sample sizes SSG is more likely to find better solutions in shorter time than MM. For small sample sizes, MM finds better solutions in shorter time than SSG on average. These observations refine the findings of Section \ref{sec:exp01}.

\paragraph*{5.} The standard deviations of SSG-1 and SSG-50 decrease faster with increasing sample size than the standard deviations of MM-1 and MM-50. This result shows that initialization of MM has a stronger effect on solution quality than the order with which SSG processes the sample time series. This finding disagrees with prior assumptions that the order of pairwise averaging causes inferior and unstable performance \cite{Petitjean2011}.

\medskip

In summary, for all but the smallest sample sizes, SSG finds better solutions than MM on average. For large sample sizes, performing a single epoch of SSG may be sufficient to find better and more robust solutions than MM-50. Moreover, for sufficiently large sample sizes SSG may converge within some small tolerance to a solution before it has processed a single epoch.

\begin{figure}[t]
\centering
\includegraphics[width=0.45\textwidth]{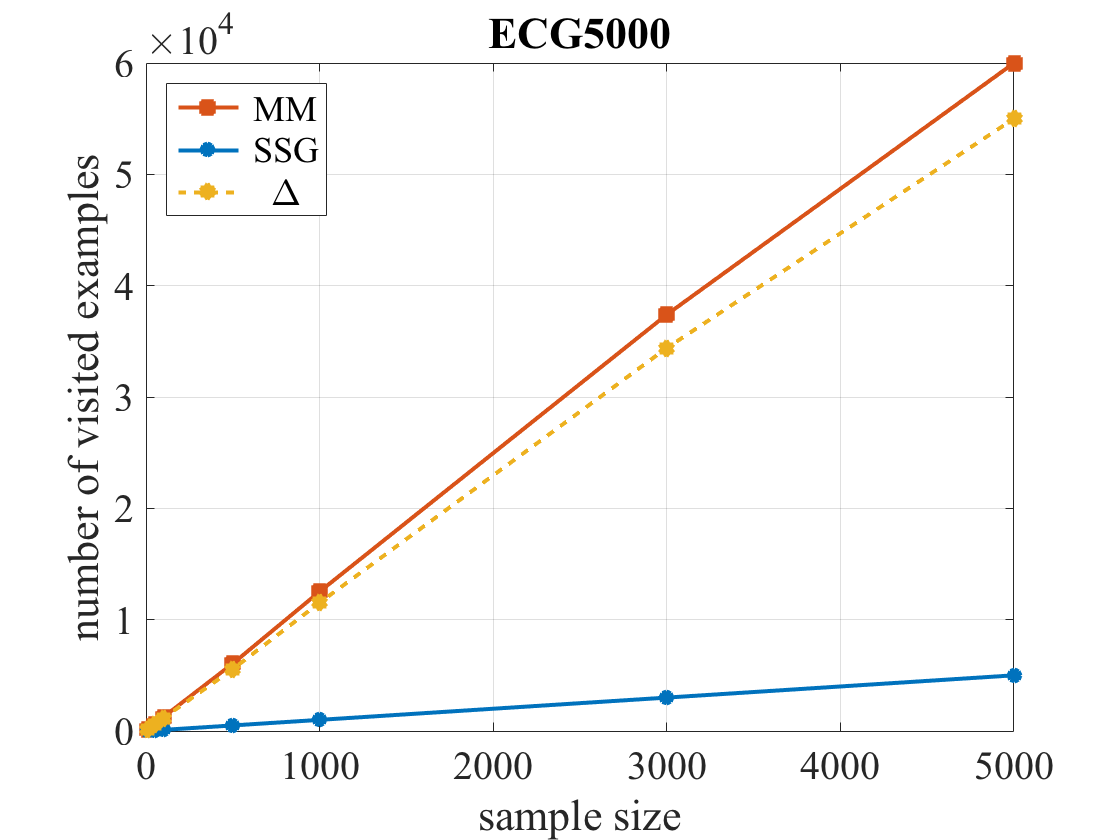}
\includegraphics[width=0.45\textwidth]{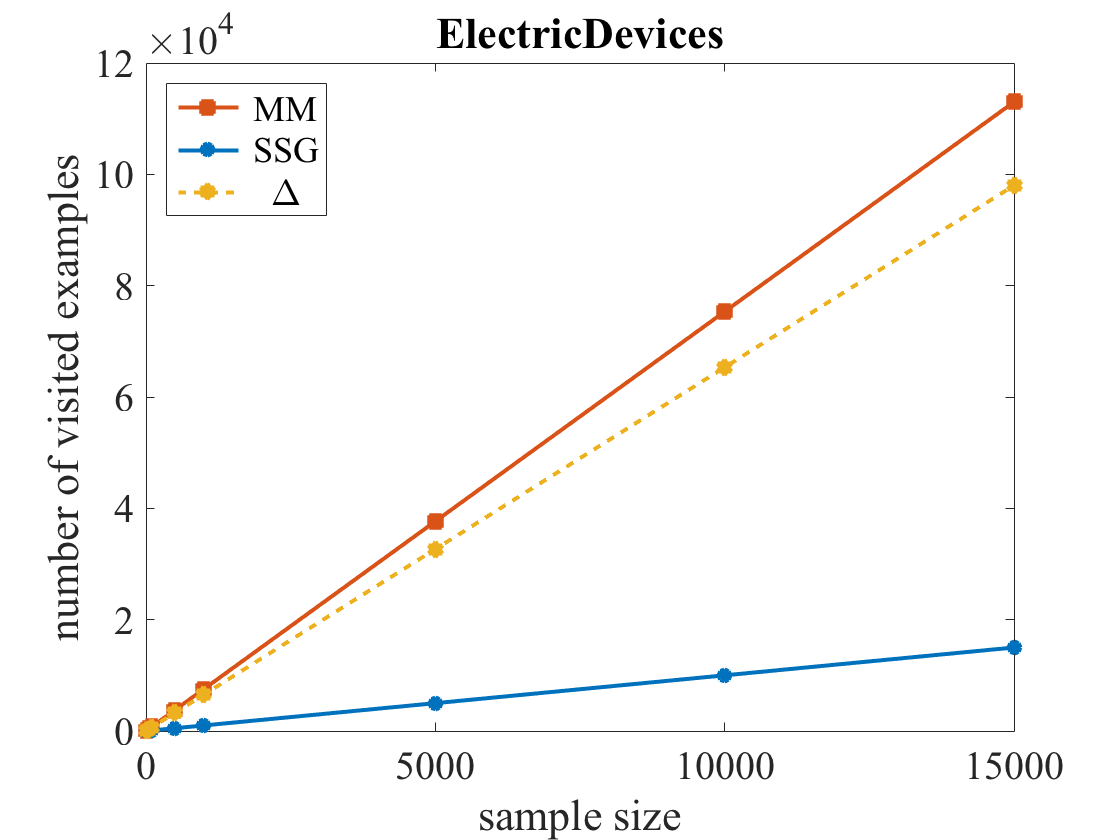}

\vspace{1em}

\includegraphics[width=0.45\textwidth]{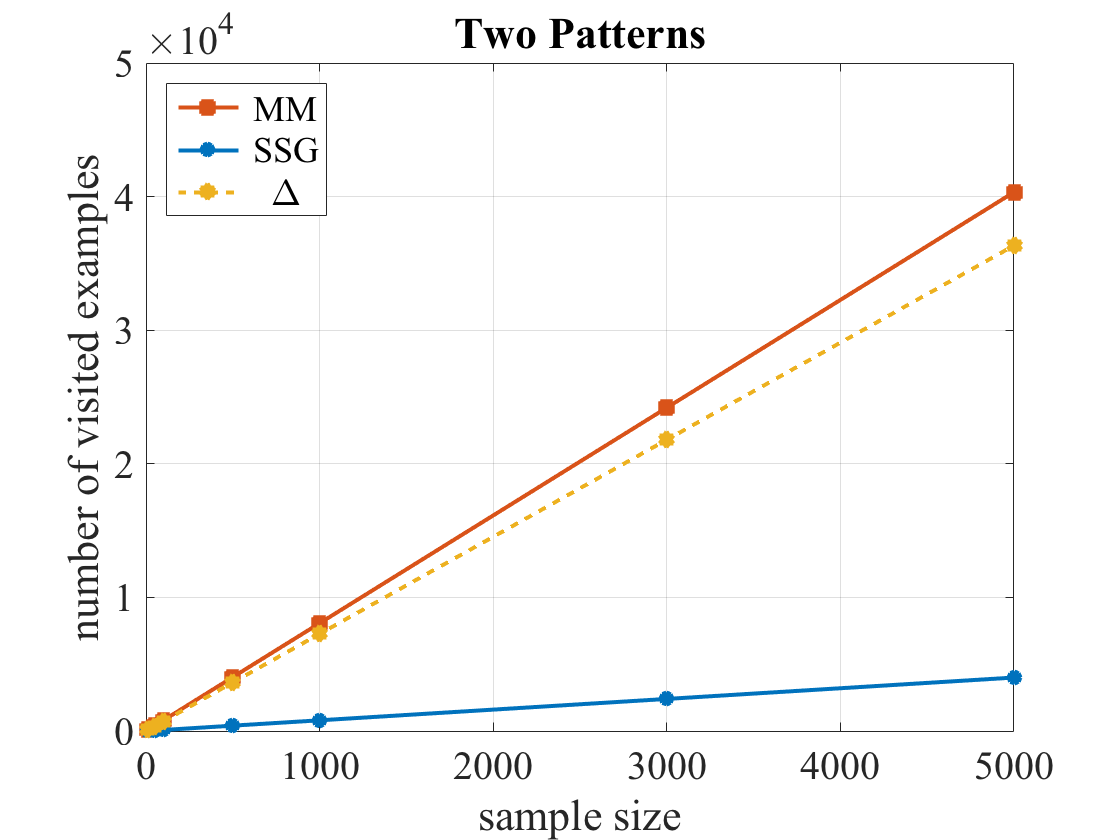}
\includegraphics[width=0.45\textwidth]{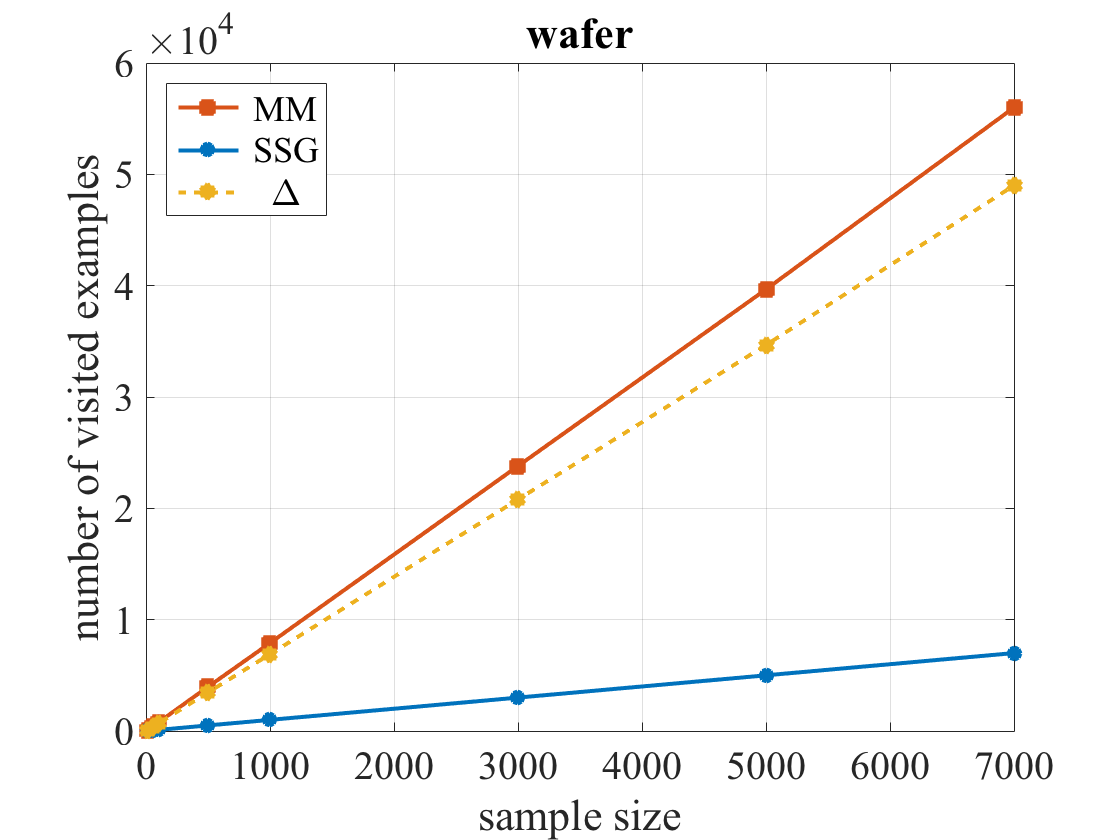}
\caption{Average number of visited examples as a function of the sample size. The considered sample sizes can be inferred by the filled balls on the lines and correspond exactly to those shown in Figure \ref{fig:ss}. Due to the scaling of the x-axis, balls of small sample sizes overlap. MM refers to MM-50, SSG refers to SSG-e$'$, where $e'$ is the number of epochs required by SSG to obtain at least the same solution quality as MM-50. Finally, $\Delta$ is the average difference of the number of visited examples between MM and SSG.}
\label{fig:sst}
\end{figure}

\subsubsection{Results on Runtime}\label{sss:results_on_runtime}

Next, we are interested in the runtimes of SSG and MM required to achieve a solution of approximately the same quality. The runtimes of SSG and MM are both $O(e N n^2)$, where $e$ is the number of epochs until termination, $N$ is the sample size, and $n$ is the length of the time series. The quadratic factor $n^2$ is caused by computing the DTW distance between a sample time series and the current solution. 
Since complexity of both algorithms is identical, it is sufficient to measure the runtimes by counting the number of visited (processed) sample time series.

As solutions of approximately the same quality, we use the variations obtained by MM-50 as reference values. The number of visited examples by MM-50 until termination at epoch $e$ is given by $eN$. Similarly, the number of visited examples by SSG-50 is $e'N$, where $e'$ is the minimum number of epochs required by SSG-50 to achieve a solution of at least the same quality as MM-50. We refer to this variant of SSG as SSG-e$'$. To ensure solutions of approximately the same quality, we excluded all trials ($\sim 5.5\%$) for which SSG-50 was unable to find a solution of at least the same quality as MM-50. 

\medskip

Figure \ref{fig:sst} summarizes the results. The results show that SSG-e$'$ found solutions of at least the same quality five up to ten times faster than MM-50. In addition, we observe that the runtimes of both algorithms increase linearly with the sample size, but with different speed. Linearity implies that the number of epochs is independent of the sample size, given that the samples are randomly drawn from the same distribution. This result confirms the observation from Section \ref{sss:results_on_variation} that many updates based on single examples lead to faster convergence to a solution of similar quality than a few updates based on the entire sample. The important advantage of SSG over MM is that computation time per update does not grow with the sample size. These findings are largely in line with those reported in the machine learning and deep learning literature on batch and stochastic gradient descent \cite{Bengio2013}. 

\subsection{Why Pairwise Averaging has Failed and how it can Succeed}

The empirical and theoretical results provide new insight on incremental (pairwise averaging) methods for the sample mean problem in DTW spaces. Empirically, we have two results:
\begin{enumerate}
\item 
The symmetric-incremental methods NLAAF \cite{Gupta1996} and PSA \cite{Niennattrakul2009} performed inferior than the asymmetric-batch MM algorithm \cite{Petitjean2011,Soheily-Khah2015}. 
\item 
For non-small sample sizes, the asymmetric-incremental SSG algorithm performed superior than the asymmetric-batch MM algorithm.
\end{enumerate} 
The first empirical result swept away incremental methods and replaced them by research on batch methods. The second empirical result disagrees with the assumption that pairwise averaging is the cause of inferior and unstable performance compared to batch methods such as the MM algorithm \cite{Petitjean2011}. Instead, we assume that two other factors cause the inferior performance of the incremental methods NLAAF and PSA:

\begin{enumerate}
\item 
Recall from Section \ref{sec:related-work} that symmetric-incremental methods result in increasingly long average time series. To adjust the length of the average time series, symmetric-incremental methods usually apply a post-processing step. This post-processing step ignores the necessary conditions of optimality. Thus it could happen that the post-processing step displaces the average far from a local minimum.

\item
The way the step size is adapted could matter. 
\end{enumerate} 

These considerations suggest to empirically reassess iterative methods such as PSA in a modified form: (i) replace symmetric averages by asymmetric averages; (ii) eventually consider a suitable schedule for adapting the step size; and (iii) optionally apply the MM algorithm after termination of the modified iterative method. 

Note that the third option can be applied to any asymmetric-iterative method including the SSG algorithm. The role of the asymmetric-iterative method can be interpreted as a sophistic way to initialize MM. Conversely, the role of the MM algorithm can be interpreted as a post-processing step to ensure the necessary conditions of optimality. 

\section{Conclusion}\label{sec:conclusion}

This article provides a theoretical foundation for the sample mean problem in DTW spaces. The main contributions can be summarized as
(1) novel directions for devising improved sample mean algorithms, (2) improved understanding of the sample mean problem, (3) improved understanding of existing mean algorithms.

Novel directions of mean algorithms are nonsmooth optimization methods and asymmetric-incremental methods. As a representative of nonsmooth optimization methods, we presented subgradient methods. As a representative of both, nonsmooth and asymmetric-incremental methods, we propose a stochastic subgradient (SSG) method. Experiments show that for increasing sample sizes the SSG algorithm is more stable and finds better solutions in shorter time than the DBA algorithm on average. Therefore, SSG is particularly useful if the sample size is non-small and in situations where several sample means have to be estimated several times such as in k-means clustering of large-scale temporal datasets. In addition, as a stochastic method, SSG can be applied in online settings, where merely one sample time series is available at a time.

Concerning contribution (2), we presented a mathematically sound way to explore analytical properties of the Fr\'echet function. This work answers the remaining of the four fundamental questions of an optimization problem (existence and uniqueness of a solution, formulation of necessary and sufficient conditions of optimality). In particular, we know the form of local minimizers and of a sample mean.

Regarding contribution (3), we showed that the DBA algorithm is a majorize-minimize algorithm and a special case of a subgradient method. In conclusion, we have a mathematically convenient formulation of DBA that simplifies analysis of the algorithm and directly leads to better suited implementations for programming languages that benefit from matrix calculus instead of loops. We used the reformulation of DBA to prove convergence to solutions satisfying the necessary conditions of optimality after a finite number of iterations and obtained a new termination criterion for DBA. Further, we showed that Soheily-Khah et al. \cite{Soheily-Khah2016} generalized DBA to weighted and kernelized DTW distances.

The theoretical and empirical results justify to further explore the two proposed directions -- nonsmooth optimization and asymmetric-incremental methods -- for devising sample mean algorithms. Methods belonging to the first direction include sophisticated nonsmooth optimization methods such as plane cutting, bundle, and gradient sampling techniques. The second direction includes asymmetric variants of the PSA and NLAAF algorithms that were originally introduced as symmetric-incremental methods.

\medskip

\paragraph*{Acknowledgements:} B.~Jain was funded by the DFG Sachbeihilfe \texttt{JA 2109/4-1}.

\begin{small}

\appendix

\newtheorem*{theorem*}{Theorem}
\newtheorem*{proposition*}{Proposition}

\section{Theoretical Background and Proofs}

This section provides the background and the proofs of the theoretical results stated in this article.

\subsection{Time Warping Embeddings}\label{ss:embeddings}

Every warping path $p \in \S{P}$ of length $L$ induces two embeddings $\Ex, \Ey:\S{T} \rightarrow \R^L$ such that the cost of aligning two time series $x$ and $y$ along warping path $p$ can be expressed as 
\begin{align}\label{eq:C_p = E}
C_p(x, y) = \normS{\Ex(x)-\Ey(y)}{^2}.
\end{align}
To construct both embeddings, we assume that $e^{k} \in \R^n$ denotes the $k$-th standard basis vector with elements
\[
e_i^k = \begin{cases}
1 & i = k\\
0 & i \neq k
\end{cases}.
\] 
\begin{definition}\label{def:embeddings}
Let $p = (p_1, \dots, p_L) \in \S{P}$ be a warping path with points $p_l = (i_l, j_l)$. Then 
\begin{equation*}
\Ex = 
\begin{bmatrix} 
e^{i_1} \\
\vdots \\
e^{i_L} \\
\end{bmatrix} \in \R^{L \times n}, \quad
\Ey = 
\begin{bmatrix} 
e^{j_1} \\
\vdots \\
e^{j_L} \\
\end{bmatrix}
 \in \R^{L \times n},
\end{equation*}
is the pair of \emph{embedding matrices} induced by warping path $p$. 
\end{definition}
The embedding matrices have full column rank $n$ due to the boundary and step condition of the warping path. Thus, we can regard the embedding matrices of warping path $p$ as injective linear maps $\Ex, \Ey:\R^n \rightarrow \R^L$ that embed time series $x$ and $y$ of length $n$ into $\R^L$ by matrix multiplication $\Ex x$ and $\Ey y$. We show that Eq.~\eqref{eq:C_p = E} holds. 

\begin{proposition}\label{prop:C=norm}
Let $\Ex$ and $\Ey$ be the embeddings induced by warping path $p \in \P$. Then 
\[
C_p(x,y) = \normS{\Ex x - \Ey y}{^2}.
\]
\end{proposition}
\begin{proof}
The assertion follows from
\begin{align*}
\normS{\Phi x - \Psi y}{^2}
= \normS{
\begin{bmatrix} 
x_{i_1} - y_{j_1} \\
\vdots \\
x_{i_L} - y_{j_L} \\
\end{bmatrix} }{^2}
= \sum_{l=1}^L \argsS{x_{i_l} - y_{j_l}}{^2}  = C_p(x, y).
\end{align*}
\end{proof}

\medskip

The next result shows that we can express the warping and valence matrix as products of the embedding matrices. 

\begin{lemma}\label{lem:properties:embeddings}
Let $\Ex$ and $\Ey$ be the embeddings induced by warping path $p \in \P$. We have
\begin{align*}
W &= \Ex \tran \Ey\\
V &= \Ex \tran \Ex,
\end{align*}
where $W$ and $V$ are the warping and valence matrix of $p$.
\end{lemma}

\begin{proof}
Let $p = \big( (i_1,j_1), \dots, (i_L,j_L) \big)$.

\medskip

\noindent
\textbf{1.} We first show $W = \Ex \tran \Ey$. Observe that
\begin{align*}
[\Ex \tran \Ey ]_{i,j} = \sum_{k=1}^L \Ex_{i,k}\tran \Ey_{k,j} = \sum_{k=1}^L \Ex_{k,i} \Ey_{k,j}
= \sum_{k=1}^L e^{i_k}_i e^{j_k}_j.
\end{align*}
Suppose that $(i,j) \in p$. Then there is a unique index $l \in [L]$ such that $(i,j) = (i_l, j_l)$, where uniqueness of $l$ follows from the step condition of a warping path. Then we have
\[
[\Ex \tran \Ey ]_{i,j} = e^{i_l}_i e^{j_l}_j = e^i_i e^j_j = 1.
\]
Now suppose that $(i,j) \notin p$. Then $(i,j) \neq (i_k, j_k)$ for all $k \in [L]$. Hence, we have $e^{i_k}_i e^{j_k}_k = 0$ for all $k \in [L]$ and therefore $[\Ex \tran \Ey ]_{i,j} = 0$. Combining the results shows that the elements of the matrix $[\Ex \tran \Ey ]$ are of the form
\[
[\Ex \tran \Ey ]_{i,j} = \begin{cases}
1 & (i,j) \in p\\
0 & \text{otherwise}
\end{cases},
\]
which precisely corresponds to the definition of a warping matrix $W$ of path $p$.

\medskip

\noindent
\textbf{2.} We show $V = \Ex \tran \Ex$. Let $u \in \R^n$ denote the vector of all ones. By definition of the valence matrix $V$, the diagonal elements of $V_{i,i}$ are of the form 
\begin{align*}
V_{i,i} 
= \sum_{j=1}^n W_{i,j} 
= \sum_{j=1}^n \,[\Ex \tran \Ey ]_{i,j} 
= \sum_{j=1}^n \sum_{k=1}^L e^{i_k}_i e^{j_k}_j 
= \sum_{k=1}^L e^{i_k}_i \sum_{j=1}^n e^{j_k}_j = \args{\sum_{k=1}^L e^{i_k}_i} \args{u \tran e^{j_k}} = \sum_{k=1}^L e^{i_k}_i. 
\end{align*}
From $e^{i_k}_i e^{i_k}_i = e^{i_k}_i$ follows
\[
V_{i,i} = \sum_{k=1}^L e^{i_k}_i e^{i_k}_i = \sum_{k=1}^L \Phi_{k,i} \Phi_{k,i} = \sum_{k=1}^L \Phi_{i,k} \tran \Phi_{k,i} = [\Ex \tran \Ex ]_{i,i}
\]
This shows the assertion for the diagonal elements of $V$. Since $V$ is a diagonal matrix, it is sufficient to show that $[\Ex \tran \Ex ]_{i,j} = 0$ for all distinct $i, j \in [n]$. Suppose that 
$i, j \in [n]$ with $i \neq j$. We have 
\[
[\Ex \tran \Ex ]_{i,j} = \sum_{k=1}^L \Phi_{k,i} \Phi_{k,j},
\]
where $\Phi_{k,i}$ and $\Phi_{k,j}$ is the $i$-th and $j$-th component of the standard basis vector $e^{i_k}$. Since $i \neq j$ by assumption, we find that $\Phi_{k,i} \Phi_{k,j} = 0$ for all $k \in [L]$. This shows that all off-diagonal elements of $[\Ex \tran \Ex ]$ are zero. By combining the results we obtain the second assertion.
\end{proof}

\subsection{Gradient of Component Functions: Proof of Prop.~\ref{prop:DF_C}}

Let $\S{X} = \args{x^{(1)},\dots,x^{(N)}}$ be a sample of $N$ time series $x^{(k)} \in \S{T}$. As shown in Section \ref{subsec:decomposition} the Fr\'echet function $F$ of $\S{X}$ is a pointwise minimum of the form
\[
F(x) = \min_{\S{C} \in \S{P}^N} F_{\S{C}}(x),
\]
with component functions 
\[
F_{\S{C}}(x) = \frac{1}{N}\sum_{k=1}^N C_{p^{(k)}}\!\args{x, x^{(k)}},
\]
where $\S{C} = \args{p^{(1)}, \dots, p^{(N)}} \in \S{P}^N$ is a configuration of warping paths associated with sample $\S{X}$. We restate Prop.~\ref{prop:DF_C}.

\begin{proposition*} Let $\S{X} = \args{x^{(1)},\dots,x^{(N)}} \in \S{T}^N$ be a sample of time series and let $\S{C} \in \S{P}^N$ be a configuration of warping paths. The gradient of component function $F_{\S{C}}$ at $x \in \S{T}$ is of the form
\begin{align*}
\nabla F_{\S{C}}(x) = \frac{2}{N} \,\sum_{k=1}^N \Big(V^{(k)}x - W^{(k)}x^{(k)}\Big),
\end{align*}
where $V^{(k)}$ and $W^{(k)}$ are the valence and warping matrix of warping path $p^{(k)} \in \S{C}$ associated with time series $x^{(k)}\in \S{X}$. 
\end{proposition*}

\begin{proof}
Let $\Ex^{(k)}$ and $\Ey^{(k)}$ denote the embeddings induced by warping path $p^{(k)} \in \S{C}$ for all $k \in [N]$. From Prop.~\ref{prop:C=norm} follows that 
\[
F_{\S{C}}(x) = \frac{1}{N} \sum_{k=1}^N \normS{\Ex^{(k)}x - \Ey^{(k)}x^{(k)}}{^2}.
\]
The function $F_{\S{C}}$ is differentiable with gradient 
\begin{align*}
\nabla F_{\S{C}}(x) &= \frac{2}{N}\sum_{k=1}^N \Ex^{(k)}{\,\tran} \args{\Ex^{(k)}x - \Ey^{(k)}x^{(k)}}\\
&= \frac{2}{N}\sum_{k=1}^N \args{\Ex^{(k)}{\,\tran}\Ex^{(k)}x - \Ex^{(k)}{\,\tran}\Ey^{(k)}x^{(k)}}\\
&= \frac{2}{N}\sum_{k=1}^N \args{V^{(k)}x - W^{(k)}x^{(k)}},
\end{align*}
where the last line follows from Lemma \ref{lem:properties:embeddings}. 
\end{proof}

\subsection{Necessary Conditions of Optimality: Proof of Theorem \ref{theorem:form}}\label{ss:proofOfMainThm}

\begin{theorem*}
Let $F$ be the Fr\'echet function of sample $\S{X} = \args{x^{(1)},\dots,x^{(N)}} \in \S{T}^N$. 
If $z \in \S{T}$ is a local minimizer of $F$, then there is a configuration $\S{C} \in \S{P}^N$ such that the following conditions are satisfied:
\begin{description}
\item[(C1)] $F(z) = F_{\S{C}}(z)$.
\item[(C2)] We have
\begin{align*}
z = \argsS{\sum_{k = 1}^N V^{(k)}}{^{-1}} \args{\sum_{k = 1}^N W^{(k)} \,x^{(k)}},
\end{align*}
where $V^{(k)}$ and $W^{(k)}$ are the valence and warping matrix of $p^{(k)}\in \S{C}$ for all $k \in [N]$. 
\end{description}
\end{theorem*}

\begin{proof}
\noindent
\textbf{1. }
We apply Prop.~\ref{prop:C=norm}, to write the Fr\'echet function $F$ of $\S{X}$ as
\[
F(x) = \frac{1}{N} \sum_{k=1}^N \min_{\S{C} } \normS{\Phi^{(k)}x - \Psi^{(k)}x^{(k)}}{^2},
\]
where the minimum is taken over all configurations $\S{C} \in \S{P}^N$. The matrices $\Phi^{(k)}$ and $\Psi^{(k)}$ are the embeddings induced by the warping paths $p^{(k)} \in \S{C}$. Suppose that $z \in \S{T}$ is a local minimizer of $F$. Then there is a configuration $\S{C}_* \in \S{P}^N$ consisting of optimal warping paths $p_*^{(k)} \in \S{P}_*\!\args{z, x^{(k)}}$ such that 
\[
F(z) = \frac{1}{N} \sum_{k=1}^N \normS{\Phi_*^{(k)}z - \Psi_*^{(k)}x^{(k)}}{^2} = F_{\S{C}_*}(z),
\]
where $\Phi_*^{(k)}$ and $\Psi_*^{(k)}$ are the embeddings induced by the optimal warping paths $p_*^{(k)} \in \S{C}_*$. The last equation shows condition (C1). 

\medskip

\noindent
\textbf{2. }
The function
\[
F_{\S{C}_*}(x) = \sum_{k=1}^N \normS{\Phi_*^{(k)}x - \Psi_*^{(k)}x^{(k)}}{^2}
\]
is convex and differentiable. By Prop.~\ref{prop:DF_C}, the gradient of $F_{\S{C}_*}(x)$ is of the form
\[
\nabla F_{\S{C}_*}(x) = \frac{2}{N} \,\sum_{k=1}^N \Big(V_*^{(k)}x - W_*^{(k)}x^{(k)}\Big),
\]
where $V_*^{(k)}$ and $W_*^{(k)}$ are the valence and warping matrix of $p_*^{(k)} \in \S{C}_*$. Setting the gradient of $F_{\S{C}_*}$ to zero and solving the equation gives
\begin{align}\label{eq:proof:theorem:form:01}
z_* = \argsS{\sum_{k = 1}^N V_*^{(k)}}{^{-1}} \args{\sum_{k = 1}^N W_*^{(k)} \,x^{(k)}}
\end{align}
as unique minimizer of $F_{\S{C}_*}(x)$. Moreover, we have 
\begin{align}
\label{eq:proof:theorem:form:02}
 F(z) &= F_{\S{C}_*}(z) \geq F_{\S{C}_*}(z_*) \\
 \label{eq:proof:theorem:form:03}
 F(x) & \leq F_{\S{C}_*}(x)
\end{align}
for all $x \in \S{T}$ by construction. 

\medskip

\noindent
\textbf{3. }
We distinguish between two cases: 
\begin{enumerate}
\item $F_{\S{C}_*}(z) = F_{\S{C}_*}(z_*)$: This directly implies $z = z_*$, because $F_{\S{C}_*}(x)$ has a unique minimizer. Then condition (C2) follows from Eq.~\eqref{eq:proof:theorem:form:01}. 
\item $F_{\S{C}_*}(z) > F_{\S{C}_*}(z_*)$: We show that this case contradicts the assumption that $z$ is a local minimizer. Let $\S{B} = \S{B}(z, \varepsilon)$ be the Euclidean ball with center $z$ and radius $\varepsilon > 0$. Then for every $\varepsilon > 0$ there is a time series $x \in \S{B}$ satisfying $F_{\S{C}_*}(x) < F_{\S{C}_*}(z)$ because $F_{\S{C}_*}$ is convex and $z$ is not a minimum of $F_{\S{C}_*}$. From Eq.~\eqref{eq:proof:theorem:form:02} and \eqref{eq:proof:theorem:form:03} follows that 
\[
 F(x) \leq F_{\S{C}_*}(x) < F_{\S{C}_*}(z) = F(z).
\]
Since $\varepsilon$ can be arbitrarily small, we find that $z$ is not a local minimizer, which contradicts our assumption. Hence, the first case applies. 
\end{enumerate}
\vspace{-3ex}
\end{proof}

\subsection{Sufficient Conditions of Optimality: Proof of Prop.~\ref{prop:sufficient-condition}}

\begin{proposition*}
Let $F$ be the Fr\'echet function of sample $\S{X}$ and let $z \in \S{T}$ be a time series. Suppose that there is a configuration $\S{C} \in \S{P}^N$ such that $z$ satisfies the necessary conditions (C1) and (C2). If $\S{C}$ is unique, then $F(z)$ is a local minimum. 
\end{proposition*}

\begin{proof}
From uniqueness of $\S{C}$ follows that the active set at $z$ is of the form $\S{A}_F(z) = \cbrace{F_{\S{C}}}$. Then $F(z) = F_{\S{C}}(z)$ is continuously differentiable in a neighborhood of $z$ with positive Hessian at $z$. Since $z$ satisfies condition (C2), the gradient of $F_{\S{C}}$ at $z$ exists and vanishes. Then the assertion follows from the sufficient conditions of a local minimum of continuously differentiable functions. 
\end{proof}

\subsection{Convergence of the MM Algorithm: Proof of Theorem \ref{theorem:convergence_of_MM}}

We first introduce some notations and definitions. By $2^{\S{X}}$ we denote the set of all subsets of some set $\S{X}$. A \emph{point-to-set map} on $\S{X}$ is a map $\alpha: \S{X} \rightarrow 2^{\S{X}}$ that assigns each point $x \in \S{X}$ to a subset $\alpha(x) \subseteq \S{X}$. Let $\S{Z}_* \subseteq \S{X}$ be a solution set. A function $f:\S{X} \rightarrow \R$ is a \emph{descent function} for $\S{Z}_*$ and $\alpha$, if 
\begin{enumerate}[(i)]
\item $x \notin \S{Z}_*$, then $f(z) < f(x)$ for all $z \in \alpha(x)$
\item $x \in \S{Z}_*$, then $f(z) \leq f(x)$ for all $z \in \alpha(x)$.
\end{enumerate}
Suppose that $\S{X} = \args{x^{(1)}, \dots, x^{(N)}}$ is a sample of $N$ time series $x^{(k)} \in \S{T}$. A configuration 
\[
\S{C} = \args{p^{(1)}, \dots, p^{(N)}} \in \S{P}^N
\] 
is an \emph{optimal configuration} for $z \in \S{T}$, if $p^{(k)} \in \S{P}_*\!\args{z, x^{(k)}}$ are optimal warping paths between $z$ and $x^{(k)} \in \S{X}$. 
By $\S{O}(z) \subseteq \S{P}^N$ we denote the set of optimal configurations for $z$. Note that $\S{C} \in \S{O}(z)$ yields $F(z) = F_{\S{C}}(z)$ by definition.
\medskip

\noindent
We restate Theorem \ref{theorem:convergence_of_MM}. 

\begin{theorem*}
The MM algorithm terminates after a finite number of iterations at a solution satisfying the necessary conditions of optimality (C1) and (C2).
\end{theorem*}

\begin{proof}
The Fr\'echet function $F$ can be written as 
\[
F(x) = \min_{\S{C} \in \S{P}^N} F_{\S{C}}(x),
\]
where $F_{\S{C}}$ are the component functions of $F$. Then the point-to-set map $\alpha$ on $\S{T}$ of the MM algorithm is of the form 
\[
\alpha(z) = \cbrace{\argmin_{x\,\in\,\S{T}} F_{\S{C}}(x) \,:\, \S{C} \in \S{O}(z)}.
\] 

\medskip

Let $\S{Z}_*$ denote the solution set consisting of all time series that satisfy conditions (C1) and (C2). We show that the Fr\'echet function $F$ is a descent function for $\S{Z}_*$ and $\alpha$. Suppose that $x \in \S{T}$. Any element $z \in \alpha(x)$ is the unique minimum of a component function $F_{\S{C}}$, where $\S{C}\in \S{O}(x)$ is an optimal configuration for $x$. We have
\begin{enumerate}
\itemsep0em
\item
 $F(z) \leq F_{\S{C}}(z)$ by definition of the Fr\'echet function. 
 \item 
 $F_{\S{C}}(z) \leq F_{\S{C}}(x)$, because $z$ is a minimum of $F_{\S{C}}$.
 \item 
 $F_{\S{C}}(x) = F(x)$, because $\S{C}$ is an optimal configuration.
\end{enumerate}
Combining the three inequalities yields $F(z) \leq F(x)$. Suppose that $x \notin \S{Z}_*$. Since $z$ is the unique minimum of the convex function $F_{\S{C}}$, we have $F_{\S{C}}(z) < F_{\S{C}}(x)$ giving $F(z) < F(x)$. This proves the properties (i) and (ii) of a descent function.

\medskip

Let $z \in \S{T}$ be the current solution. Then the next solution $z' \in \alpha(z)$ is the unique minimum of a component function $F_{\S{C}}$, where $\S{C} \in \S{O}(z)$ is an optimal configuration for $z$. Suppose that the MM algorithm terminates at $z'$. We show that $z' \in \S{Z}_*$. Observe that 
\[
F(z') \leq F_{\S{C}}(z') = \min_x F_{\S{C}}(x) \leq F(z).
\]
Termination is enforced by $F(z') = F(z)$. In this case, we have $F(z') = F_{\S{C}}(z')$, which is condition (C1). Condition (C2) holds by construction of the map $\alpha$.

\medskip

It remains to show that the MM algorithm terminates after a finite number of iterations. Let $\args{z^{(t)}}$ be a sequence with $z^{(t+1)} \in \alpha\!\args{z^{(t)}}$. Suppose that $\args{\S{C}^{(t)}}$ is a sequence of optimal configurations $\S{C}^{(t)} \in \S{O}\!\args{z^{(t)}}$ for $z^{(t)}$ such that 
\[
z^{(t+1)} = \argmin_{x \in \S{T}} F_{\S{C}^{(t)}}(x).
\]
We assume that the sequence $\args{z^{(t)}}$ is infinite. Since $F$ is a descent function and the MM algorithm does not terminate, we have $F\!\args{z^{(s)}} < F\!\args{z^{(t)}}$ for all $s, t \in \N$ such that $0 \leq t < s$. Observe that 
\[
F\!\args{z^{(s+1)}} \leq F_{\S{C}^{(s)}}\!\args{z^{(s+1)}} \leq F\!\args{z^{(s)}} \leq F\!\args{z^{(t+1)}} \leq F_{\S{C}^{(t)}}\!\args{z^{(t+1)}} \leq F\!\args{z^{(t)}}. 
\]
Suppose that $\S{C}^{(s)} = \S{C}^{(t)}$. Since the minimum of component function $F_{\S{C}^{(t)}}$ is unique, we find that $z^{(s+1)} = z^{(t+1)}$. This yields $F\!\args{z^{(s+1)}} = F\!\args{z^{(t+1)}}$, which contradicts our assumption that $F$ is strictly decreasing. This shows that $\S{C}^{(s)} \neq \S{C}^{(t)}$ for all $s > t$. Hence, the configurations of the sequence $\args{\S{C}^{(t)}}$ are mutually distinct. This contradicts our assumption that the sequence $\args{z^{(t)}}$ is infinite, because there are only finitely many different configurations. Consequently, the MM algorithm terminates after finitely many iterations to a solution satisfying the necessary conditions of optimality. 
\end{proof}

\section{Generalizations}\label{sec:generalizations}

In the main text, we considered the special case of univariate sample time series of fixed length. This section briefly sketches in two steps how to generalize the results of the univariate-fixed case. The first step relaxes the condition that the univariate sample time series are of fixed length. We call this case the univariate-variable case. The second step relaxes the condition that the time series are univariate, called the multivariate-variable case. 

\subsection{Generalization to the Univariate-Variable Case}\label{subsec:univariate-variable}

The univariate-variable case assumes that $F :\S{T}_n \rightarrow \R$ is a Fr\'echet function of sample $\S{X} = \args{x^{(1)}, \dots, x^{(N)}}$ consisting of $N$ univariate time series $x^{(k)} \in \S{T}_{n_k}$ of possibly variable length $n_k \in \N$. 

\medskip

Suppose that $x \in \S{T}_n$ is the argument of the Fr\'echet function and $y \in \S{T}_{m}$ is a sample time series, that is $y = x^{(k)}$ and $m = n_k$ for some $k \in [N]$. The results carry over to the univariate-variable case for the following reasons: 

\smallskip

\noindent
1. The cost $C_p(x, y)$ of aligning $x$ and $y$ along warping path $p \in \S{P}_{n, m}$ can be regarded as a differentiable convex function on $\R^n$. By adopting the same arguments as in Section \ref{sec:properties}, we find that the Fr\'echet function $F$ of sample $\S{X}$ is also a pointwise minimum of a set of differentiable convex functions and is therefore locally Lipschitz. Hence, the subdifferential calculus is also applicable to the univariate-variable case. 

\smallskip

\noindent
2. The embedding matrices $\Phi \in \R^{L \times n}$ and $\Psi \in \R^{L \times m}$ induced by a warping path $p \in \S{P}_{n,m}$ of length $L$ can be defined in a similar way as in Definition \ref{def:embeddings}. What has changed is that the rows of $\Psi$ have the form of standard basis vectors from $\R^m$ rather than from $\R^n$.

\smallskip

\noindent
3. The statement of Lemma \ref{lem:properties:embeddings} and its proof carry over to the univariate-variable case. The warping matrix $W = \Phi^{\T}\Psi$ is a well-defined ($n \times m$)-matrix and the valence matrix $V = \Phi^{\T}\Phi$ is a well-defined ($n \times n$)-matrix. The proof of Lemma \ref{lem:properties:embeddings} requires some noncritical adaptions to account for the length $m$ of the sample time series $x$. Apart from this, the proof remains the same. 

\smallskip

\noindent
4. By construction, the expressions $Vx$ and $Wy$ are well-defined vectors in $\R^n$ regardless of the length $m$ of sample time series $y$. Consequently, the expression
\[
\sum_{k=1}^N V^{(k)}x-W^{(k)}x^{(k)}
\] 
is a well-defined vector in $\R^n$, where the $x^{(k)}$ are the sample time series of sample $\S{X}$. The proofs of Prop.~\ref{prop:DF_C} and Theorem \ref{theorem:form} directly carry over to the univariate-variable case. The gradient of $F$ and condition (C2) of Theorem \ref{theorem:form} are essentially of the same form as in the univariate-fixed case. What differs are the dimensions of the warping matrices. 
 
\smallskip

\noindent
5. The proofs of all other results are independent of the length of the sample time series.

\subsection{Generalization to the Multivariate-Variable Case}

The multivariate-variable case assumes that $F :\S{T}_n \rightarrow \R$ is a Fr\'echet functions of sample $\S{X} = \args{x^{(1)}, \dots, x^{(N)}}$ consisting of $N$ multivariate time series $x^{(k)} \in \S{T}_{n_k}$ of possibly variable length $n_k \in \N$. 

\medskip

A $d$-dimensional time series $x$ of length $n$ can be represented as a matrix $x = (x_{i,j}) \in \R^{n \times d}$, where $d \geq 1$. Every column $x_{:j}$ of matrix $x$ corresponds to a univariate \emph{component time series} of the form
\begin{align*}
x_{:j} = (x_{1j}, \dots, x_{nj}) \in \R^n
\end{align*}
for all $j \in [d]$. Moreover, every row $x_i$ of matrix $x$ represents a $d$-dimensional feature vector 
\begin{align*}
x_{i} = (x_{i1}, \dots, x_{id}) \in \R^d.
\end{align*}
for every time point $i \in [n]$. 

We assume that all multivariate time series under considerations are of dimension $d$. Then by $\S{T}_n$ we denote the set of multivariate time series of length $n$. Suppose that $x \in \S{T}_n$ is the argument of the Fr\'echet function and $y \in \S{T}_{m}$ is a sample time series, that is $y = x^{(k)}$ and $m = n_k$ for some $k \in [N]$. The results carry over to the multivariate-variable case for the following reasons: 

\smallskip

\noindent
1. Let $p = (p_1, \dots, p_L) \in \S{P}_{n, m}$ be a warping path between $x$ and $y$ with elements $p_l = (i_l, j_l)$ for all $l \in [L]$. The cost $C_p(x, y)$ of aligning $x$ and $y$ along warping path $p$ is of the form
\[
C_p(x, y) = \sum_{l = 1}^L \normS{x_{i_l}-y_{j_l}}{^2}.
\]
By assuming that $y$ is a fixed parameter, we can regard the cost $C_p$ as a differentiable convex function on $\R^{n \times d}$. 
Then as in Section \ref{subsec:univariate-variable}, we find that the Fr\'echet function $F$ of sample $\S{X}$ is locally Lipschitz continuous as a pointwise minimum of a set of differentiable convex functions. Therefore, the subdifferential calculus is also applicable to the multivariate-variable case. 

\smallskip

\noindent
2. Next, we construct the embeddings of a warping path. Suppose that $p = (p_1, \dots, p_L) \in \P_{n, m}$ is a warping path between $x$ and $y$. The multivariate time warping embeddings $\Phi': \S{T}_n \rightarrow \S{T}_L$ and $\Psi' : \S{T}_m \rightarrow \S{T}_L$ induced by $p$ are linear functions defined by
\begin{align*}
\Phi'(x) = \begin{bmatrix} 
x_{i_1} \\
\vdots \\
x_{i_L} \\
\end{bmatrix} 
,\qquad \Psi'(y) = 
\begin{bmatrix} 
y_{j_1} \\
\vdots \\
y_{j_L} \\
\end{bmatrix}.
\end{align*}
We regard $\Phi'$ and $\Psi'$ as linear mappings $\Phi':\R^{n \times d} \rightarrow \R^{L \times d}$ and $\Psi':\R^{m \times d} \rightarrow \R^{L \times d}$. Then the multivariate formulation of Prop.~\ref{prop:C=norm} remains true with Frobenius norm $\norm{(x_{i,j})}_F = \left( \sum_{i=1}^N \sum_{j=1}^d x_{i,j}^2 \right)^{1/2}$:
\[
C_p(x,y) = \norm{\Ex'( x) - \Ey'( y)}_F^2.
\]

\smallskip

\noindent
3. The operations of $\Phi'$ and $\Psi'$ induced by warping path $p \in \P_{n,m}$ can be expressed by componentwise matrix multiplications with the embedding matrices (cf.~Section \ref{ss:embeddings}) $\Phi$ and $\Psi$ induced by $p$:
\begin{align*}
\Phi' (x)  =  \begin{bmatrix} \Phi x_{:1} & \dots  &\Phi x_{:d} \end{bmatrix},
\qquad \Psi'(y)  =  \begin{bmatrix} \Psi y_{:1} & \dots  &\Psi y_{:d} \end{bmatrix}.
\end{align*}
Therefore, componentwise, everything reduces to the univariate-variable case. This is sufficient, because gradients are defined componentwise and the necessary conditions of optimality are derived from gradients. Hence, all other results carry over to the multivariate-variable case. 

\smallskip

\noindent
4. When considering the problem componentwise, it is important to note that the valence and warping matrices $V^{(k)}$ and $W^{(k)}$, resp., are induced by warping paths  $p^{(k)} \in \P_{n,n_k}$ of a \emph{single} configuration $\S{C} = \args{p^{(1)}, \dots, p^{(N)} }$ and therefore identical for each component time series.
For example the gradient $\nabla F_{\S{C}} (x) \in \R^{n \times d}$ of a component function $F_C$ at some $x \in \T_n$ is given by
\begin{align*}
\nabla F_{\S{C}}(x) = \frac{2}{N}  \sum_{k=1}^N
\begin{bmatrix}
 \Big(V^{(k)}x_{:1} - W^{(k)}x_{:1}^{(k)}\Big) & \dots & \Big(V^{(k)}x_{:d} - W^{(k)}x_{:d}^{(k)}\Big)
\end{bmatrix}.
\end{align*}
\end{small}

\end{document}

%% file: config.tex
\usepackage{amsmath,amssymb,color,epsfig,fullpage} 
\usepackage{amsthm}
\usepackage{verbatim,fancyvrb}
\usepackage{alltt}
\usepackage{tikz}
\usepackage{enumerate}
\usepackage{caption}
\usepackage{subfigure}
\usepackage[]{algorithm}
\usepackage{algpseudocode}
\usepackage{stmaryrd}
\usepackage{afterpage}
\usepackage{url}
\usepackage{framed}

\newtheorem{definition}{Definition}[section]
\newtheorem{theorem}[definition]{Theorem}
\newtheorem{corollary}[definition]{Corollary}
\newtheorem{proposition}[definition]{Proposition}
\newtheorem{lemma}[definition]{Lemma}

\def\N{{\mathbb N}}

\def\<#1,#2>{\langle#1,#2\rangle} 

\def\R{{\mathbb R}}
\def\T{{\mathcal T}}

\def\P{{\mathcal P}}

\newcommand{\Oh}{\mathcal O}

\DeclareMathOperator*{\argmin}{argmin}

\DeclareMathOperator{\dtw}{dtw}

\providecommand{\norm}[1]{\left\Vert #1 \right\Vert}

\makeatletter
\def\section{\@startsection{section}{1}{\z@}{-3.5ex plus -1ex minus
 -.2ex}{2.3ex plus .2ex}{\normalsize\bf}}
\def\subsection{\@startsection{subsection}{2}{\z@}{-3.25ex plus -1ex
 minus -.2ex}{1.5ex plus .2ex}{\normalsize\bf}}
\makeatother

\newcommand*{\tran}{^{\mkern-1.5mu\mathsf{T}}}

\newcommand*{\Ex  }{\Phi}
\newcommand*{\Ey}{\Psi}

\newcommand{\commentout}[1]{}

\newcommand{\abs}[1]{\mathop{\left\lvert #1 \right\rvert}} 
\newcommand{\args}[1]{\mathop{\left( #1 \right)}} 

\newcommand{\cbrace}[1]{\mathop{\left\{ #1 \right\}}}

\newcommand{\argsS}[2]{\mathop{\left( #1 \right)#2}} 

\newcommand{\normS}[2]{\mathop{\left\lVert #1 \right\rVert#2}}

\renewcommand{\S}[1]{{\mathcal{#1}}}           	

\renewenvironment{cases}{%
\left\{\begin{array}{c@{\quad : \quad}l}}%
{%
\end{array}\right.}

\algblock{OptParFor}{EndOptParFor}
\algnewcommand\algorithmicparfor{\textbf{for}}
\algnewcommand\algorithmicpardo{\textbf{do (optionally in parallel threads)}}
\algnewcommand\algorithmicendparfor{\textbf{end\ for}}
\algrenewtext{OptParFor}[1]{\algorithmicparfor\ #1\ \algorithmicpardo}
\algrenewtext{EndOptParFor}{\algorithmicendparfor}

%% file: DTW_Mean.bbl
\section*{References}
\begin{thebibliography}{00}
\setlength{\parskip}{0pt}
\setlength{\itemsep}{0pt plus 0.3ex}
\small
\bibitem{Abdulla2003}
W.H.~Abdulla, D.~Chow, and G.~Sin. 
\newblock Cross-words reference template for DTW-based speech recognition systems. 
\newblock \emph{Conference on Convergent Technologies for Asia-Pacific Region}, 2003.

\bibitem{Bagirov2014}
A.~Bagirov, N.~Karmitsa, and M.M.~M\"akel\"a.
\newblock \emph{Introduction to Nonsmooth Optimization}.
\newblock Springer International Publishing, 2014.

\bibitem{Bengio2013}
Y.~Bengio.
\newblock Practical Recommendations for Gradient-Based Training of Deep Architectures.
\newblock \emph{Neural Networks: Tricks of the Trade}, K.-R.~M\"uller, G.~Montavon, and G.B.~Orr (eds.), Springer, 2013.

\bibitem{Bonnans2006}
J.F. Bonnans, J.C. Gilbert, C. Lemarechal, and C.A. Sagastizbal. 
\newblock \emph{Numerical optimization: theoretical and practical aspects}. Springer Science+Business Media, 2006.

\bibitem{Chen2015}
Y.~Chen, E.~ Keogh, B.~Hu, N.~Begum, A.~Bagnall, A.~Mueen, and G.~Batista.
\newblock \emph{The UCR Time Series Classification Archive}.
\newblock URL: \url{www.cs.ucr.edu/~eamonn/time_series_data/}, 2015.

\bibitem{Clarke1990}
F.H.~Clarke. 
\newblock \emph{Optimization and Nonsmooth Analysis}. 
\newblock Society for Industrial and Applied Mathematics, 1990.

\bibitem{Dimitriadou2002}
E.~Dimitriadou, A.~Weingessel, and K.~Hornik. 
\newblock A combination scheme for fuzzy clustering. 
\newblock \emph{International Journal of Pattern Recognition and Artificial Intelligence}, 16(07):901--912, 2002.

\bibitem{Ermoliev1998}
Y.~Ermoliev and V.~Norkin.
\newblock Stochastic generalized gradient method for nonconvex nonsmooth stochastic optimization.
\newblock \emph{Cybernetics and Systems Analysis}, 34(2):196--215, 1998.

\bibitem{Evans1992}
L.C.~Evans and R.F.~Gariepy. 
\newblock \emph{Measure theory and fine properties of functions}.
\newblock CRC Press, 1992.

\bibitem{Frechet1948}
M.~Fr\'{e}chet.
\newblock Les \'el\'ements al\'eatoires de nature quelconque dans un espace distanci\'e.
\newblock \emph{Annales de l'institut Henri Poincar\'e}, 215--310, 1948.

\bibitem{Gupta1996}
L.~Gupta, D.~Molfese, R.~Tammana, and P.G.~Simos.
\newblock Nonlinear alignment and averaging for estimating the evoked potential.
\newblock \emph{IEEE Transactions on Biomedical Engineering}, 43(4):348--356, 1996.

\bibitem{Gusfield1997}
D.~Gusfield.
\newblock \emph{Algorithms on strings, trees and sequences: computer science and computational biology}.
 \newblock Cambridge University Press, 1997.

\bibitem{Hautamaki2008}
V.~Hautamaki, P.~Nykanen, P.~Franti.
\newblock Time-series clustering by approximate prototypes.
\newblock \emph{International Conference on Pattern Recognition}, 2008.

\bibitem{Hunter2004}
D.~Hunter and K.~Lange. 
\newblock A tutorial on MM algorithms.
\newblock \emph{The American Statistician }, 58(1):30-37, 2004.

\bibitem{Jain2016a}
B.J.~Jain.
\newblock Statistical Analysis of Graphs.
\newblock \emph{Pattern Recognition}, 60:802--812, 2016.

\bibitem{Jain2016b}
B.J.~Jain and D.~Schultz.
\newblock A Reduction Theorem for the Sample Mean in Dynamic Time Warping Spaces.
\newblock \emph{arXiv preprint}, arXiv:1610.04460, 2016.

\commentout{
\bibitem{Juang1990}
B.H.~Juang and L.R.~Rabiner. 
\newblock Hidden Markov models for speech recognition. 
\newblock \emph{Technometric}, 33(3), 251--272, 1990.
}

\bibitem{Lummis1973}
R.~Lummis.
\newblock Speaker verification by computer using speech intensity for temporal registration. 
\newblock \emph{IEEE Transactions on Audio and Electroacoustics}, 21(2):80--89, 1973. 

\bibitem{Kruskal1983}
J.B.~Kruskal and M.~Liberman. 
\newblock The symmetric time-warping problem: From continuous to discrete.
\newblock \emph{Time warps, string edits and macromolecules: The theory and practice of sequence comparison}, pp.
125--161, 1983.

\bibitem{Niennattrakul2009}
V.~Niennattrakul and C.A.~Ratanamahatana.
\newblock Shape averaging under time warping. 
\newblock \emph{IEEE International Conference on Electrical Engineering/Electronics, Computer, Telecommunications and Information Technology}, 2009. 

\bibitem{Niennattrakul2012}
V.~Niennattrakul, D.~Srisai, and C.A.~Ratanamahatana.
\newblock Shape-based template matching for time series data
\newblock \emph{Knowledge-Based Systems}, 26:1--8, 2012.

\bibitem{Norkin1986}
V.~Norkin.
\newblock Stochastic generalized-differentiable functions in the problem of nonconvex nonsmooth stochastic optimization.
\newblock \emph{Cybernetics and Systems Analysis}, 22(6):804--809, 1986.

\bibitem{Oates1999}
T.~Oates, L.~Firoiu, and P.R.~Cohen. 
\newblock Clustering Time Series with Hidden Markov Models and Dynamic Time Warping. 
\newblock IJCAI Workshop on Neural, Symbolic and Reinforcement Learning methods for Sequence Learning, 1999.

\bibitem{Ongwattanakul2009}
S.~Ongwattanakul and D.~Srisai. 
\newblock Contrast enhanced dynamic time warping distance for time series shape averaging classification. 
\newblock \emph{International Conference on Interaction Sciences: Information Technology, Culture and Human}, 2009.

\bibitem{Petitjean2011}
F.~Petitjean, A.~Ketterlin, and P.~Gancarski. 
\newblock A global averaging method for dynamic time warping, with applications to clustering.
\newblock \emph{Pattern Recognition} 44(3):678--693, 2011.

\bibitem{Petitjean2012}
F.~Petitjean and P.~Gan{\c{c}}arski. 
\newblock Summarizing a set of time series by averaging: From Steiner sequence to compact multiple alignment.
\newblock \emph{Theoretical Computer Science}, 414(1):76--91, 2012.

\bibitem{Petitjean2014Code}
F. Petitjean.
\newblock Matlab and Java source code for DBA.
\newblock doi:10.5281/zenodo.10432, 2014.

\bibitem{Petitjean2016}
F.~Petitjean, G.~Forestier, G.I.~Webb, A.E.~Nicholson, Y.~Chen, and E.~Keogh.
\newblock Faster and more accurate classification of time series by exploiting a novel dynamic time warping averaging algorithm. 
\newblock \emph{Knowledge and Information Systems}, 47(1):1--26, 2016.

\bibitem{Rabiner1979}
L.R.~Rabiner and J.G. Wilpon.
\newblock Considerations in applying clustering techniques to speaker-independent word recognition. 
\newblock \emph{The Journal of the Acoustical Society of America}, 66(3): 663--673, 1979.

\bibitem{Sakoe1978}
H.~Sakoe and S.~Chiba. 
\newblock Dynamic programming algorithm optimization for spoken word recognition. 
\newblock \emph{IEEE Transactions on Acoustics, Speech, and Signal Processing}, 26(1):43--49, 1978.

\bibitem{Schultz2016Code}
D.~Schultz and B.J.~Jain.
\newblock Sample Mean Algorithms for Averaging in Dynamic Time Warping Spaces.
\newblock 10.5281/zenodo.216233, 2016.

\bibitem{Shor1985}
N.Z.~Shor. 
\newblock \emph{Minimization Methods for Non-Differentiable Functions}. 
\newblock Springer-Verlag Berlin, Heidelberg, 1985.

\bibitem{Soheily-Khah2015}
S.~Soheily-Khah, A.~Douzal-Chouakria, and E.~Gaussier.
\newblock Progressive and Iterative Approaches for Time Series Averaging. 
\newblock \emph{Workshop on Advanced Analytics and Learning on Temporal Data}, 2015.

\bibitem{Soheily-Khah2016}
S.~Soheily-Khah, A.~Douzal-Chouakria, and E.~Gaussier.
\newblock Generalized k-means-based clustering for temporal data under weighted and kernel time warp.
\newblock \emph{Pattern Recognition Letters}, 75:63--69, 2016.

\bibitem{Somervuo1999}
P.~Somervuo and T.~Kohonen. 
\newblock Self-organizing maps and learning vector quantization for feature sequences. 
\newblock \emph{Neural Processing Letters}, 10(2):151--159, 1999.

\bibitem{Srisai2009}
D.~Srisai and C.A.~Ratanamahatana. 
\newblock Efficient time series classification under template matching using time warping alignment. 
\newblock \emph{IEEE International Conference on Computer Sciences and Convergence Information Technology}, 2009. 

\bibitem{Zangwill1969}
W.~Zangwill.
\newblock \emph{Nonlinear Programming: A Unified Approach},
\newblock Prentice-Hall, 1969.
\end{thebibliography}
